\documentclass[journal,twoside,web]{ieeecolor}
\usepackage{generic}
\usepackage{cite}
\usepackage{amsmath,amssymb,amsfonts}
\usepackage{algorithmic}
\usepackage{graphicx}
\usepackage{array,makecell,booktabs,multirow,amssymb}
\usepackage{algorithm,algorithmic}
\usepackage{hyperref}
\hypersetup{hidelinks=true}
\usepackage{textcomp}
\usepackage{mathtools}

\newtheorem{proposition}{Proposition}
\def\BibTeX{{\rm B\kern-.05em{\sc i\kern-.025em b}\kern-.08em
    T\kern-.1667em\lower.7ex\hbox{E}\kern-.125emX}}
\begin{document}
\title{TA-LSDiff:Topology--Aware Diffusion Guided by a Level Set Energy for Pancreas Segmentation}

\author{Yue Gou, Fanghui Song, Yuming Xing, Shengzhu Shi, Zhichang Guo and Boying Wu
\thanks{This work is supported in part by the National Natural Science Foundation of China (12171123, 12271130, 12371419, 12401557, U21B2075).(Corresponding authors: Fanghui Song.)}
\thanks{Yue Gou, Fanghui Song, Yuming Xing, Shengzhu Shi, Zhichang Guo\\
and Boying Wu are with the Department of Computational Mathematics,\\
School of Mathematics, Harbin Institute of Technology, Harbin, 150001,\\
China. }}

\maketitle
\begin{abstract}
Pancreas segmentation in medical image processing is a persistent challenge due to its small size, low contrast against adjacent tissues, and significant topological variations. Traditional level set methods drive boundary evolution using gradient flows, often ignoring pointwise topological effects. Conversely, deep learning–based segmentation networks extract rich semantic features but frequently sacrifice structural details. To bridge this gap, we propose a novel model named TA-LSDiff, which combined topology--aware diffusion probabilistic model and level set energy, achieving segmentation without explicit geometric evolution. This energy function guides implicit curve evolution by integrating the input image and deep features through four complementary terms. To further enhance boundary precision, we introduce a pixel--adaptive refinement module that locally modulates the energy function using affinity weighting from neighboring evidence. Ablation studies systematically quantify the contribution of each proposed component. Evaluations on four public pancreas datasets demonstrate that TA-LSDiff achieves state-of-the-art accuracy, outperforming existing methods. These results establish TA-LSDiff as a practical and accurate solution for pancreas segmentation. 
\end{abstract}
\section{Introduction}
\label{sec:introduction}
\IEEEPARstart{A}{ccurate} pancreas segmentation is a critical prerequisite in medical imaging, essential for tasks such as surgical planning and volumetric assessment of lesions. Despite its clinical significance, achieving precise segmentation is notoriously challenging \cite{wong2021implementation,sung2021global,khristenko2021preoperative}. The pancreas is characterized by its small size, low contrast, and high inter--patient variability in both morphology and topology \cite{zhang2024state}. Currently, clinicians manually delineate the pancreas on computed tomography (CT) scans, a workflow that is both time--consuming and prone to subjective influence. These limitations highlight the urgent clinical need for robust, automated segmentation methods that can deliver reliable results, thereby enhancing diagnostic accuracy and streamlining clinical workflows \cite{podinua2025artificial}.

To address these issues, researchers have explored both traditional and deep learning--based methods. Traditional image segmentation is dominated by model--driven approaches. Among these, variational level set methods \cite{osher1988fronts,adalsteinsson1995fast,kass1988snakes,mumford1989optimal} are highly popular due to their flexibility and mathematical rigor. These methods implicitly represent contours via a high--dimensional energy function, which is iteratively minimized using gradient descent. However, these traditional frameworks lack semantic information, leading to clear limitations when confronted with the complex anatomy of the pancreas. In addition, variational level set methods require manual parameter tuning, which is particularly challenging in abdominal CT imaging due to its high structural complexity.The pancreas varies significantly in shape and location across patients, meaning a fixed set of parameters often yields suboptimal results. This variability hinders the practical adoption of such methods in clinical applications.

With the rapid improvement of deep learning, data--driven methods such as convolutional neural networks (CNNs) have shown great success in medical image segmentation due to their strong ability to learn features from data \cite{zhu2005semi,roth2018deep}.
Given the pancreas's anatomical complexity, most studies leverage the U-Net \cite{ronneberger2015u} architecture and its variants. These models employ a symmetric encoder--decoder structure with skip connections to enhance detailed feature representation.
A primary research focus has been on refining the feature extraction power of this core architecture. For instance, Oktay et al. \cite{oktay2018attention} proposed Attention U-Net, which uses attention mechanisms to selectively emphasize pancreas--relevant regions. Other variants, such as MBU-Net \cite{huang2022semantic} or ResDAC-Net \cite{ji2024resdac}, incorporated lightweight backbones, dilated convolutions, or residual asymmetric kernels to improve both efficiency and feature representation. These architecture--centric improvements achieved solid baseline performance, reaching Dice scores in the 82--84\%.

However, simple feature extraction struggles with the pancreas's variable shape and low contrast. Therefore, research has focused on capturing broader contextual and spatial information. Models like ADAU-Net \cite{li2022attention} incorporated pyramid pooling modules to aggregate multi--scale features, while PBR-UNet \cite{li2021pancreas} proposed a hybrid regularization scheme to fuse inter--slice probability maps. More recently, approaches like CT-UNet \cite{chen2023ctunet} have combined 3D convolutions with Transformers to explicitly model volumetric context. These efforts, along with others employing advanced cross--domain connections \cite{li2020multiscale} or fuzzy skip logic \cite{chen2022target}, successfully pushed segmentation accuracy to nearly 88\% Dice.

However, purely data--based methods suffer from fundamental drawbacks: they require large--scale datasets with high labeling costs and often operate as "black-box". They learn complex patterns but operate without explicit geometric rules or topological constraints \cite{muhammad2024unveiling}. This lack of built--in structural knowledge can lead to results with subtle topological errors (e.g., disconnected components, incorrect holes) or rough, implausible boundaries. This "black-box" nature limits their practical use and acceptance in real clinical environments, underscoring the need for a new paradigm that integrates semantic learning with geometric and topological interpretability.

This limitation of purely data--driven models has naturally led researchers to explore hybrid frameworks that integrate the semantic power of deep learning method with the geometric topological nature of level set methods. These approaches generally fall into two categories. The first category uses deep learning as a powerful feature extractor to guide classic level set evolution. For instance, Level Set R-CNN \cite{homayounfar2020levelset} incorporates Chan--Vese (CV) evolution into Mask R-CNN. Similarly, BoxLevelSet \cite{li2022box,li2024box2mask} demonstrates that a network's bounding--box proposals can be refined by CV evolution to achieve segmentation from weak supervision. The second category attempts to make the level set function itself end-to-end differentiable. This is achieved either by training networks to learn the level set function directly \cite{hu2017deep} or by formulating the loss based on variational principles, such as the Mumford-Shah functional \cite{kim2019mumford}.

These hybrid methods have shown potential in pancreatic segmentation, with approaches combining 3D U-Net and level set evolution \cite{tian2023two,gou2025diffusion} or using geodesic distance priors \cite{hu2020automatic} pushing Dice scores to around 89\%. However, there is a key drawback: geometric evolution is typically treated as a computationally intensive post--processing step, relying on explicit partial differential equation (PDE) solvers \cite{tian2023two}, rather than being integrated into the network's inference process. Meanwhile, a powerful new generative paradigm has emerged for segmentation: denoising diffusion probabilistic model (DDPM) \cite{amit2021segdiff, wu2022medsegdiff}. The core strength of DDPM is their geometry--adaptive log--domain smoothing \cite{farghly2025diffusion}. Unlike data--domain smoothing, this approach naturally interpolates along the manifold's tangent directions rather than smearing probability mass off--manifold. As a result, intermediate predictions already inhabit a plausible shape manifold. However, relying solely on this implicit geometric bias is insufficient for sparse--data tasks like pancreas CT segmentation. Data sparsity leads to multiple possible interpolation manifolds, requiring explicit guidance to select the correct one. Therefore, a mild force can be applied without significant cost to guide the generation process toward the correct topology and a smoother boundary. The level set energy in TA-LSDiff serves this exact purpose: it provides the explicit geometric bias required to define the target manifold, guiding the model to select one with a reasonable topological structure and geometric smoothness.

In summary, deep learning excels at recognizing the overall object but struggles with fine details and lacks interpretability. Conversely, traditional variational level set methods are more interpretable and inherently robust to topological changes, but they often require manual adjustments and are computationally slow. This highlights a clear need for a new approach that combines the interpretability and topological guarantees of traditional methods with the semantic power of deep learning. Inspired by these methods, we propose TA-LSDiff, a topology--ware diffusion model. Our method operates by minimizing a level set energy, enabling topology comprehension while bypassing complex geometric calculations. This energy function comprises four key terms: a region term to distinguish the pancreas from its surroundings; a length term to maintain boundary smoothness; an area term to control the segmented size; and a distance penalty term to enforce precise localization and prevent background leakage. Furthermore, we introduce a pixel--adaptive refinement module that uses information from neighboring pixels to enhance the stability and precision of boundary decisions. This synergistic approach is significantly better at preserving fine structural details and the correct topology of the pancreas.

The main contributions are as follows:

\begin{enumerate}
    \item We establish a theoretical link between the CV $L_2$ gradient flow and the boundary topological derivative. Leveraging this insight, we propose TA-LSDiff, a topology--aware diffusion framework that markedly improves segmentation accuracy. 

    \item We design a level set energy functional that integrates four complementary terms to guide network optimization. This energy is directly injected into the reverse diffusion process to guide its optimization path and constrain the final segmentation.

    \item  We introduce a pixel--adaptive refinement (PAR) module that operates within the level set framework to locally modulate the energy. This module improves intra--regional consistency and enforces smooth boundaries, resulting in precise and coherent contours.
\end{enumerate}

The remaining parts of this article are organized as follows: Section II reviews preliminaries and Section III discusses our motivation. Section IV describes the TA-LSDiff method and its technical details. Section V presents the experimental setup, results, and analysis. Finally, Section VI concludes the paper.

\section{Preliminary}

\textbf{Notation.} Let $\Omega\subset\mathbb{R}^2$ be a bounded image domain and $I:\Omega\to\mathbb{R}$ represent a CT slice. The goal is to find a binary segmentation $y:\Omega\to\{0,1\}$ of the pancreas. In practice, this ideal binary mask is approximated by a relaxed mask $y\in[0,1]$. Note we will use $x \in \Omega$ to denote the spatial coordinates when calculating topological derivatives.

\subsection{Level Set Method and CV Model}
The level set method represents the region of interest by a level set function $\phi:\Omega\to\mathbb{R}$, where the boundary $\Gamma$ is defined as $\Gamma=\{y\in\Omega:\phi(y)=0\}$, the interior $\{\phi>0\}$ and the exterior $\{\phi<0\}$. To manipulate regions and their boundaries in a differentiable way, a smooth approximation method based on the Heaviside function $H$ and its derivative, the Dirac delta $\delta=H'$ can be used.
A popular choice is
\[
H(s)=\tfrac{1}{2}\Big(1+\tfrac{2}{\pi}\arctan(\tfrac{s}{\varepsilon})\Big),\quad
\delta(s)=\tfrac{1}{\pi}\tfrac{\varepsilon}{\varepsilon^2+s^2}.
\]
This defines the region indicators $\chi_1(\phi)=H(\phi)$ (inside) and
$\chi_2(\phi)=1-H(\phi)$ (outside). Consequently, the relaxed mask and level set are linked by $y\approx H(\phi)$.

The evolution of the level set function $\phi$ is governed by the following general form:
\begin{equation}
    \frac{\partial\phi}{\partial t} + F |\nabla\phi| = 0,
\end{equation}
where $F$ represents driving force function for evolution. This force function is typically determined by image features and segmentation task requirements.

The Chan-Vese (CV) model \cite{chan2001active} represents a significant milestone in the development of level set methods. As a region--based active contour, its core advantage stems from the fundamental assumption of approximate intensity homogeneity within the foreground and background regions. The CV functional is defined as:
\begin{equation}
\label{eq:cv-energy}
\mathcal{E}_{\mathrm{CV}}(\Omega_1,\Omega_2;c_1,c_2)
=
\int_{\Omega_1} (f(y)-c_1)^2\,\mathrm{d}y
+\int_{\Omega_2} (f(y)-c_2)^2\,\mathrm{d}y.
\end{equation}
Its corresponding level set form is:
\begin{equation}
\begin{split}
\label{eq:cv-phi-energy}
\mathcal{E}_{\rm CV}(\phi;c_1,c_2)
&=\int_\Omega (f-c_1)^2\,H(\phi)\,\mathrm{d}y \\
&+\int_\Omega (f-c_2)^2\bigl(1-H(\phi)\bigr)\,\mathrm{d}y.
\end{split}
\end{equation}
For a fixed $\phi$, the optimal region means, $c_1$ and $c_2$, have closed--form solutions: 
\begin{equation}
\label{eq:c1c2}
c_1(\phi)=\frac{\int_\Omega f\,H(\phi)\,\mathrm{d}y}{\int_\Omega H(\phi)\,\mathrm{d}y},\qquad
c_2(\phi)=\frac{\int_\Omega f\,[1-H(\phi)]\,\mathrm{d}y}{\int_\Omega [1-H(\phi)]\,\mathrm{d}y}.
\end{equation}
where \(c_1\) and \(c_2\) are the mean intensities inside (\(\phi>0\)) and outside (\(\phi<0\)) the contour, respectively. The functional is performed via a $L_2$ gradient flow, \(\partial_{t}\phi=-\,\delta \mathcal{E}_{\mathrm{CV}}/\delta\phi\), which yields the following PDE:
\begin{equation}
\label{eq:cv-flow}
\partial_{t}\phi
=
\delta(\phi)\Big[
-(f(x)-c_1)^2
+(f(x)-c_2)^2
\Big].
\end{equation}
This flow deforms the boundary \(\Gamma\) along the steepest--descent direction.

To control the smoothness of the contour, a length term $\int_{\Omega}\big|\nabla H(\phi)\big|\,\mathrm{d}x$ can be added; Similarly, a distance--regularized term, $\frac{1}{2}\!\int_\Omega (|\nabla\phi|-1)^2\,\mathrm{d}x$, is often employed for numerical stability and to avoid frequent reinitialization to a signed--distance function. Its variation contributes a diffusion--like operator to improve time--step stability. The level set method allows for the incorporation of different functional terms into the evolution equation based on task requirements and prior knowledge (e.g., shape, position, or constraints).

\subsection{Topological Derivative in Level set Segmentation}
\label{subsec:td}

The topological derivative (TD) of a shape functional $F(\Omega)$ at $x\in\Omega$ quantifies the functional's sensitivity to the nucleation of a small inclusion \cite{he2007incorporating,he2007solving}:
\begin{equation}
\label{topo}
d_TF(\Omega)(x) \;:=\;
\lim_{\rho\to 0}\frac{F(\Omega\setminus B_{\rho,x})-F(\Omega)}{|B_{\rho,x}\cap\Omega|}.
\end{equation}
A negative $d_TF<0$ indicates that creating a small hole/inclusion near $x$ decreases the energy, which in turn suggests that a topology change (e.g., splitting, merging, or hole creation) is favorable. Unlike pure boundary evolution, TD provides a nonzero driving field even away from the boundary $\Gamma$, thus reducing dependence on initialization and enabling early global topology exploration. For region--driven functionals such as \eqref{eq:cv-energy}, the boundary variation and the TD agree in sign along $\Gamma$. This ensures consistency between local interface motion and global topology updates.
\subsection{DDPM Framework for Segmentstion}
We employ a denoising diffusion probabilistic model \cite{ho2020denoising,amit2021segdiff,wu2022medsegdiff,rahman2023ambiguous} to impose a data--driven prior on relaxed masks $y\in[0,1]^\Omega$. The forward process is a fixed Markov chain that gradually adds Gaussian noise to a clean segmentation mask \(y_0\) over \(T\) steps, following a variance schedule \(\{\beta_t \in (0,1)\}_{t=1}^T\). 
\begin{equation}
\label{eq:forward-diff}
p(y_t\mid y_0)=\mathcal{N}\!\big(y_t;\,\sqrt{\bar{\alpha}_t}\,y_0,\,(1-\bar{\alpha}_t) I\big).
\end{equation}
Let \(\alpha_t = 1-\beta_t\) and \(\bar{\alpha}_t = \prod_{s=1}^t \alpha_s\).

The reverse process aims to learn the transition \(p_\theta(y_{t-1}|y_t)\) to gradually remove the noise. 
\begin{equation}
y_{t-1}=\frac{1}{\sqrt{\alpha_t}}\left(y_t-\frac{1-\alpha_t}{\sqrt{1-\bar{\alpha}_t}} \varepsilon_\theta\left(y_t,I, t\right)\right)+\sigma_t\xi_t, \quad \xi_t \sim \mathcal{N}(0, I).
\end{equation}

The core insight of DDPM is that this can be achieved by training a network, \(\epsilon_\theta(y_t, I, t)\), to predict the noise component \(\epsilon\) from the noisy input \(y_t\). Sampling starts from pure noise \(y_T \sim \mathcal{N}(0, I)\) and iteratively denoises the sample using the learned network.
Among them, $\epsilon_\theta$ represents the noise distribution learned by the network, and its objective function is the noise added during the forward process:
\begin{equation}
\label{ddpm-target}
\theta^* = \operatorname*{argmin}_\theta \mathbb{E}_{t,y_0,\epsilon} \left[ w_t \left\| \epsilon - \epsilon_\theta(y_t, I, t) \right\|_2^2 \right].
\end{equation}
Here, the optimal $\epsilon_\theta(y_t, I, t)$ is almost exactly the same as $\epsilon$ in all cases. The objective in \eqref{ddpm-target} can also be restated as an equivalent score--matching objective \cite{song2020score}:
\begin{equation}
\begin{split}
\theta^* = \operatorname*{argmin}_\theta
    \sum_{t=1}^{N} \sigma_t^2 & \mathbb{E}_{p_{\text{data}}} \mathbb{E}_{p(y_t \mid y_0)} \left[ \left\| s_\theta(y_t, I, t) \right. \right. \\
    & \left. \left. - \nabla_{y_t} \log p(y_t \mid y_0) \right\|_2^2 \right]. \label{eqn:ddpm_obj}
\end{split}
\end{equation}
Here, $p_{data}$ represents the data distribution, and the optimal $s_\theta(y_t, I, t)$ almost perfectly matches $\nabla_{y_t} \log p(y_t\mid y_0)$ in all cases.

\section{Motivation}
Our core motivation is to establish a general framework that unifies the data--driven priors of diffusion models with the geometric of variational level set methods. Our solution is to translate the model--driven principles into a language that the data--driven model can understand: gradients.

\subsection{The Chan-Vese Model as a Canonical Example}
We show a sign--consistency (hence a proportionality up to a positive factor) between the $L_2$ gradient-flow speed of the CV model and a boundary topological driving term.

\begin{proposition}
\textbf{(Boundary consistency of CV descent and topological drive.)}
\label{prop:cv_td_equivalence}
Consider the Chan--Vese energy \eqref{eq:cv-energy} with its topological derivative:
\begin{equation}
\label{topo-cv}
    \mathcal{T}_{\rm CV}(x)=-(f(x)-c_1)^2+(f(x)-c_2)^2.
\end{equation}
Then on the zero level set $\Gamma=\{\phi=0\}$ the normal velocity induced by $L_2$ gradient descent satisfies
\begin{equation}
\big\langle -\tfrac{\delta\mathcal{E}_{\rm CV}}{\delta\phi},\,n\big\rangle\Big|_{\Gamma}
\;=\; \frac{\delta_\varepsilon(0)}{|\nabla\phi|}\,\mathcal{T}_{\rm CV}(x),
\end{equation}
hence it has the same sign as $\mathcal{T}_{\rm CV}(x)$.
Moreover, the boundary topological derivative (computed by nucleating an infinitesimal inclusion at $x\in\Gamma$) has a first--order variation proportional to $\mathcal{T}_{\rm CV}(x)$, and thus also shares its sign.
\end{proposition}

\begin{proof}
(1) \emph{\textbf{Boundary topological derivative.}}
Nucleating/removing an infinitesimal ball $B_{\rho,x}$ centered at $x_0\in\Gamma$ swaps its label (inside/outside) to first order. The energy change is
\begin{equation}
\begin{split}
\Delta \mathcal{E}_{\rm CV} = \left[-(f(x_0)-c_1)^2 + (f(x_0)-c_2)^2\right]|B_{\rho,x}| + o(|B_{\rho,x}|).
\end{split}
\end{equation}
If $c_1,c_2$ are recomputed after the perturbation, their changes are $O(|B_{\rho,x}|)$, contributing only $o(|B_{\rho,x}|)$ to the first--order term. According to the definition of the topological derivative \eqref{topo}, we can easily obtain \eqref{topo-cv} $\mathcal{T}_{\rm CV}(x)$. The detailed derivation is provided in Appendix~\ref{app:cv_proof}.

(2) \emph{\textbf{Variational gradient on the boundary.}}
Taking the topological derivative in $\phi$ yields
\begin{equation}
-\frac{\delta\mathcal{E}_{\rm CV}}{\delta\phi}
=\delta_\varepsilon(\phi)\,\big[-(f-c_1)^2+(f-c_2)^2\big]
=\delta_\varepsilon(\phi)\,\mathcal{T}_{\rm CV}(x).
\end{equation}
The $L_2$ gradient flow $\partial_\tau\phi=-\delta\mathcal{E}/\delta\phi$ and the level set $V_n=-\partial_\tau\phi/|\nabla\phi|$ give, on $\Gamma$,
\begin{equation}
V_n(x)=\frac{\delta_\varepsilon(0)}{|\nabla\phi(x)|}\,\mathcal{T}_{\rm CV}(x),
\end{equation}
where $\delta_\varepsilon(0)>0$. Hence $V_n(x)$ is positively proportional to $\mathcal{T}_{\rm CV}$ and has the same sign.
\end{proof}

\paragraph{Remarks.}
(i) If a perimeter term $\mu\!\int\delta_\varepsilon(\phi)|\nabla\phi|\,\mathrm{d}x$ (or a balloon term $\nu\!\int H(\phi)\mathrm{d}x$) is included, the boundary flow adds $\mu\,\kappa+\nu$,
\[
V_n \;\propto\; \underbrace{\mathcal{T}_{\rm CV}(x)}_{\text{region drive}} 
\;+\; \underbrace{\mu\,\kappa + \nu}_{\text{geometric}},
\]
and the corresponding boundary drive should be augmented accordingly.  

(ii) This sign--consistency shows that local steepest descent on $\Gamma$ agrees with the global energetic preference of an infinitesimal topology change at $\Gamma$, providing a clean motivation to fuse boundary evolution with topological cues in our method.

\subsection{DDPM's Log-Domain Smoothing} 

We adopt the DDPM framework because its learning target is the score function \(s_\theta(y,t)\approx\nabla\log p_t(y)\)~\cite{hyvarinen2005estimation}. Crucially, and any translation--invariant linear smoothing of the score is exactly equivalent to smoothing the log--density itself:
\begin{equation}
  (k \ast s_t)(y)
  \;=\; (k \ast \nabla \log p_t)(y)
  \;=\; \nabla\!\big(k \ast \log p_t\big)(y),
\end{equation}
where \(\ast\) denotes convolution. Thus DDPMs effectively perform log--domain smoothing rather than density--domain smoothing. Recent theory shows that such log-domain smoothing is geometry--adaptive~\cite{farghly2025diffusion}.
In the small--noise regime one can approximate:
\begin{equation}
  \log p_t(y)
  \;\approx\;
  \log p_{\mathcal M}\!\big(\pi_{\mathcal M}y\big)
  \;-\;
  \frac{1}{2\,\sigma(t)^2}\,\mathrm{dist}\!\big(y,\mathcal M\big)^2
  \;+\; \mathrm{const}.
\end{equation}
This means smoothing predominantly propagates along the tangent directions of the data manifold $\mathcal M$, while sharply penalizing normal excursions. As a result, intermediate predictions (e.g., masks \(y_t\)) naturally remain near a plausible shape manifold. This allows TA-LSDiff to impose only mild topological and boundary regularization to steer the samples toward the correct topology.

\begin{figure*}
  \centering
  \hspace*{-0.3cm}\includegraphics[width=7.2in, trim=-0.4in 0 0 0, clip]{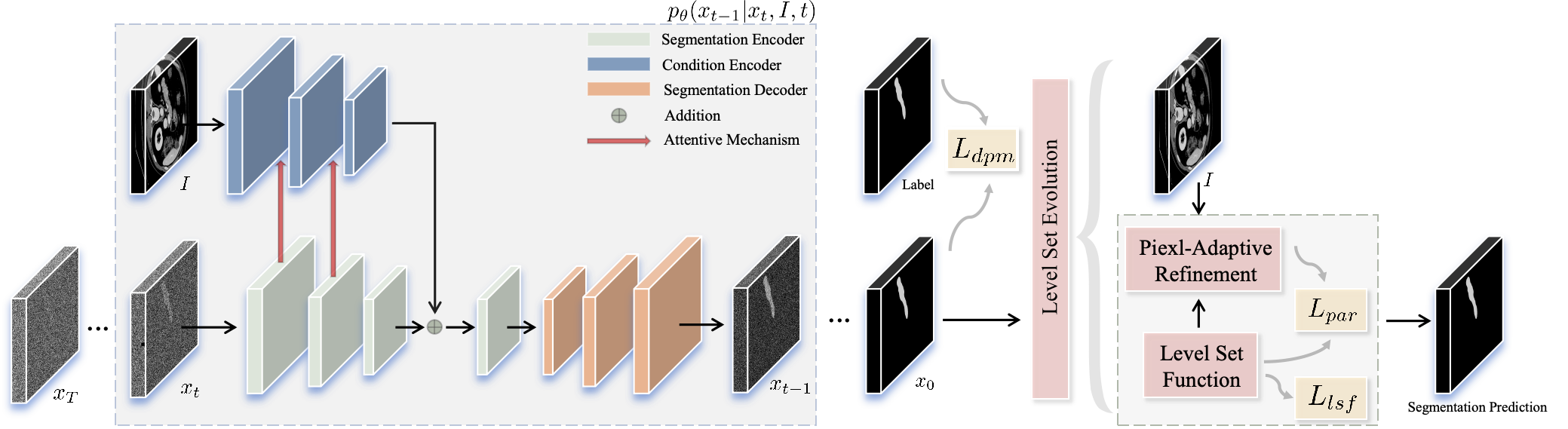}
  \caption{The blue box represents the processing procedure of the diffusion probability model, and we use the ResNet encoder and U-Net decoder to achieve this process. The encoder consists of a group of conditional encoders and a segmentation encoder with an attention mechanism on the feature fusion path. The green box represents the processing procedure of the level set energy and pixel adaptive module.}
  \label{network}
\end{figure*}

\subsection{A General Framework for Energy--Guided Diffusion}
\label{subsec:energy-motivation}

We re--interpret the level set energy \(\mathcal{E}_{\mathrm{var}}\), as a variational prior within a Bayesian context. We formalize this by defining a posterior probability distribution for a segmentation \(y
_t\) as proportional to its exponentiated negative energy:
\begin{equation}
\label{eq:posterior-var}
p(y_t\mid I)\;\propto\;\exp\!\big(-\mathcal{E}_{\mathrm{var}}(\phi(y_t);I)\big).
\end{equation}
This probabilistic view provides a profound insight and forms the theoretical bridge connecting the two paradigms. The score function $\nabla_{y_t}\log p(y_t|I)$, which points in the direction of steepest probability ascent, is precisely the negative variational gradient of the energy $\mathcal{E}_{var}$, which points in the direction of steepest energy descent:
\begin{equation}
\label{eq:score-energy-identity}
\nabla_{y_t} \log p(y_t\mid I) \;\propto\; -\,\nabla_{y_t} \mathcal{E}_{\mathrm{var}}(\phi(y_t);I).
\end{equation}
This identity enables us to directly translate geometric and topological properties into guiding signals for the diffusion process.

\subsection{Practical Implications for Our Design}
\label{subsec:practical}
Classic variational energies, while theoretically coherent, often lack the semantic expressiveness to avoid suboptimal local minima in noisy, low--contrast CT volumes. Furthermore, their gradient flow (shape derivative) is typically confined to the object's boundary $\Gamma$ \cite{chan2001active}.

Conversely, Diffusion Probabilistic Models learn a powerful semantic prior, providing a point--wise score $s_{\theta}(y_t, I, t)$ across the entire domain $\Omega$.
However, they lack explicit geometric constraints, which can lead to subtle topological errors or rough boundaries \cite{gou2025diffusion}.

Our hybrid scheme reframes this segmentation task as a functional gradient descent process, synergizing the strengths of both paradigms.
This process is guided simultaneously by two distinct functional gradients:
\begin{enumerate}
\item \textbf{A Data--Driven Gradient}: The Diffusion Probabilistic Model's learned score function $s_{\theta}(y_t, I, t)$, which approximates the data log--likelihood gradient $\nabla \log p(y_t \mid I)$.
    
\item \textbf{A Geometric--Constraint Gradient}: Inspired by the Topological Derivative, we leverage its key property: the TD defines a scalar field over the entire domain $\Omega$, quantifying the global energy's sensitivity to a point--wise topological perturbation ('flipping' a pixel) at any location $x$. This all--domain, point--wise mechanism is naturally aligned with the DDPM's score function. Based on this alignment, we design a topology--aware level set energy, $L_{lsf}$ (detailed in Sec.\eqref{sec:learnable-energy}). Following \eqref{eq:score-energy-identity}, its negative gradient, $-\nabla L_{lsf}$, serves as the geometric gradient derived from this topology--aware prior.
\end{enumerate}

During inference, this synergy is realized as an energy-guided sampling process. We combine these two functional gradients to define a regularized score, $\hat{s}_{\theta}$, which guides the segmentation process:
\begin{equation}
\hat{s}_{\theta}(y_t, I, t) = \underbrace{s_{\theta}(y_t, I, t)}_{\text{Data Gradient}} - \underbrace{\gamma_{st} \nabla_{y_t} L_{lsf}(\phi(y_t); I)}_{\text{Geometric Gradient}}.
\label{eq:score_guidance}
\end{equation}
Equation~\eqref{eq:score_guidance} formulates the guidance in the score space.
Leveraging the relationship between the score and noise predictions ($s_\theta = -\epsilon_\theta / \sqrt{1-\overline{\alpha}_t}$), this guidance can be equivalently expressed in the noise prediction space as:
\begin{equation}
\hat{\epsilon}_{\theta}(y_t, I, t) = \underbrace{\epsilon_{\theta}(y_t, I, t)}_{\text{Data Prediction}} + \underbrace{\gamma_{\epsilon t} \nabla_{y_t} L_{lsf}(\phi(y_t); I)}_{\text{Geometric Gradient}},
\label{eq:noise_guidance}
\end{equation}
where $\gamma_{st}$ and $\gamma_{\epsilon t}$ are guidance scales. 

Our approach integrates model--driven concepts directly into the data--driven framework, achieving robust and fine--grained boundary delineation while bypassing the need for computationally expensive PDE solvers.

This constitutes the core motivation for the TA-LSDiff design.

\section{TA-LSDiff} 
In this section, we propose TA-LSDiff, a novel topology--aware diffusion probabilistic model. As illustrated in Fig. \ref{network}, our approach introduces a level set energy function to guide network learning and a pixel--adaptive refinement module for local refinement. This combination aims to achieve topologically consistent and detail--preserving pancreatic segmentation.

\subsection{Level Set Energy for Guiding the Diffusion Probabilistic Model}
\label{sec:learnable-energy}
Recent research has focused on integrating diverse information sources within variational frameworks for image segmentation. In this context, our proposed energy functional integrates four complementary terms:

\begin{equation}
\begin{split}
\label{total-variational}
L_{lsf}=\mathcal{E}(\phi;I)
&= \lambda_1\,\underbrace{\mathcal{E}_{\mathrm{Region}}(\phi;I)}_{\text{regional statistics}} + \lambda_2\,\underbrace{\mathcal{E}_{\mathrm{Length}}(\phi)}_{\text{smoothness}} \\
&\quad+\lambda_3\,\underbrace{\mathcal{E}_{\mathrm{Area}}(\phi)}_{\text{size prior}} + \lambda_4\,\underbrace{\mathcal{E}_{\mathrm{Distance}}(\phi;I)}_{\text{localization}},
\end{split}
\end{equation}
where $\lambda_1,\lambda_2,\lambda_3,\lambda_4>0$ are the weight of different terms.

\subsubsection{Region Term}
We employ the Gaussian model to represent the conditional density of each region, in order to describe visual consistency and distinguishability. The negative log--likelihood function is defined as:
\begin{equation}
\label{Gauss-region}
\begin{aligned}
\mathcal{E}_{\mathrm{Region}}&(\Omega_1,\Omega_2; I)= \int_{\Omega_1} e_1(y)\,\mathrm{d}y
          + \int_{\Omega_2} e_2(y)\,\mathrm{d}y,\\
e_i(y) &= \log\lvert\Sigma_i\rvert
        + \big(I(y)-\mu_i\big)^{\top}\Sigma_i^{-1}\big(I(y)-\mu_i\big).
\end{aligned}
\end{equation}
The level set form of \eqref{Gauss-region} is
\begin{equation}
\label{eq:E-region}
\mathcal{E}_{\mathrm{Region}}(\phi;I)=
\int_{\Omega} \Big[e_1(y)\,H(\phi)+e_2(y)\,\big(1-H(\phi)\big)\Big]\,\mathrm{d}y.
\end{equation}
The Euler–Lagrange equations for \(\mu_i,\Sigma_i\) have closed-form solutions conditioned on \(\phi\):
\begin{figure}
  \centering
  \hspace*{-0.8cm}\includegraphics[width=3.8in, trim=-0.4in 0 0 0, clip]{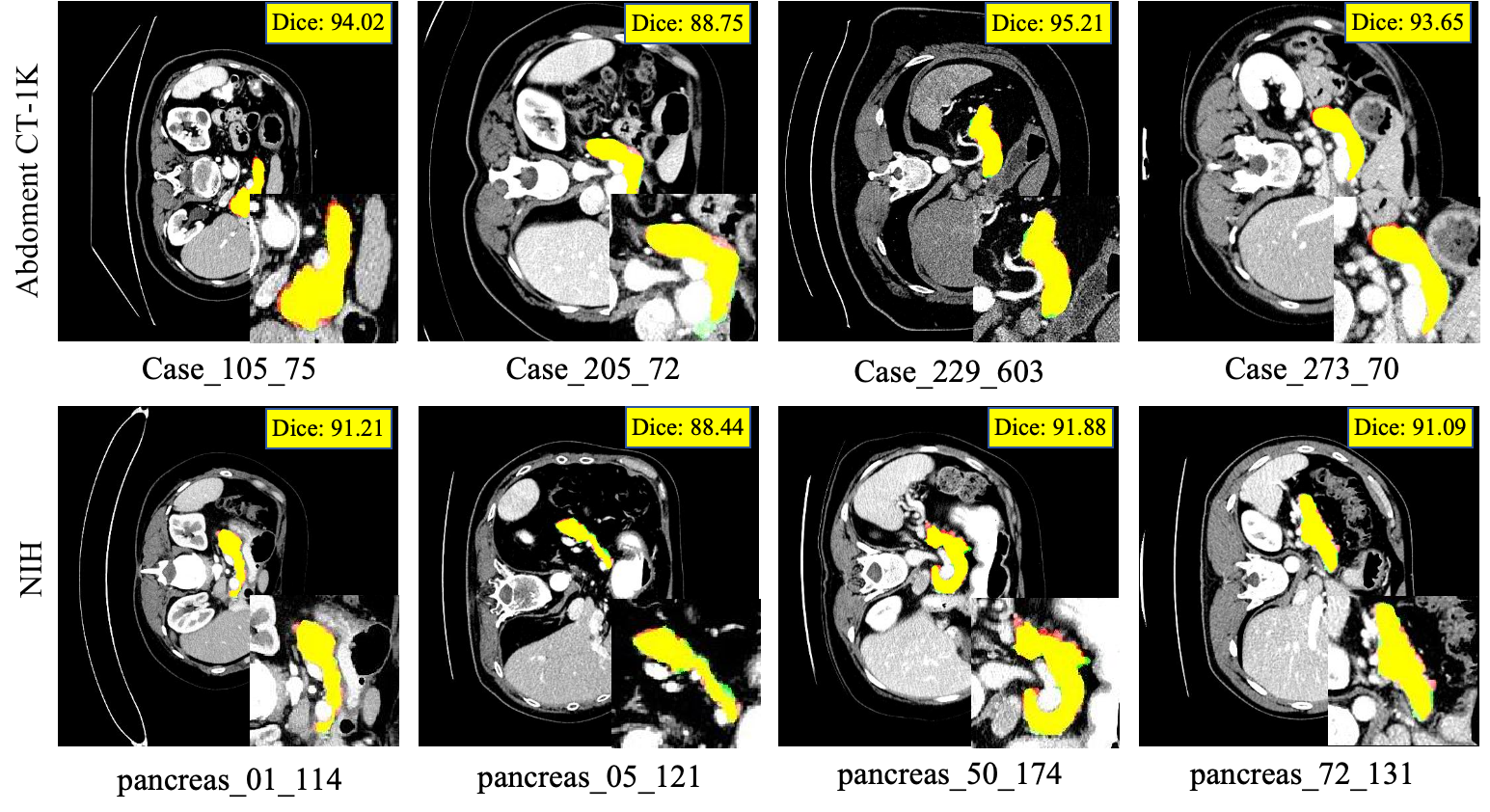}
  \caption{Segmentation results on the AbdomenCT-1K dataset (first row) and the NIH dataset (second row). Contours: Red = Gold Standard, Green = Predicted Result.}
  \label{abknih}
\end{figure}
\begin{figure}
  \centering
\hspace*{-0.8cm}\includegraphics[width=3.8in, trim=-0.4in 0 0 0, clip]{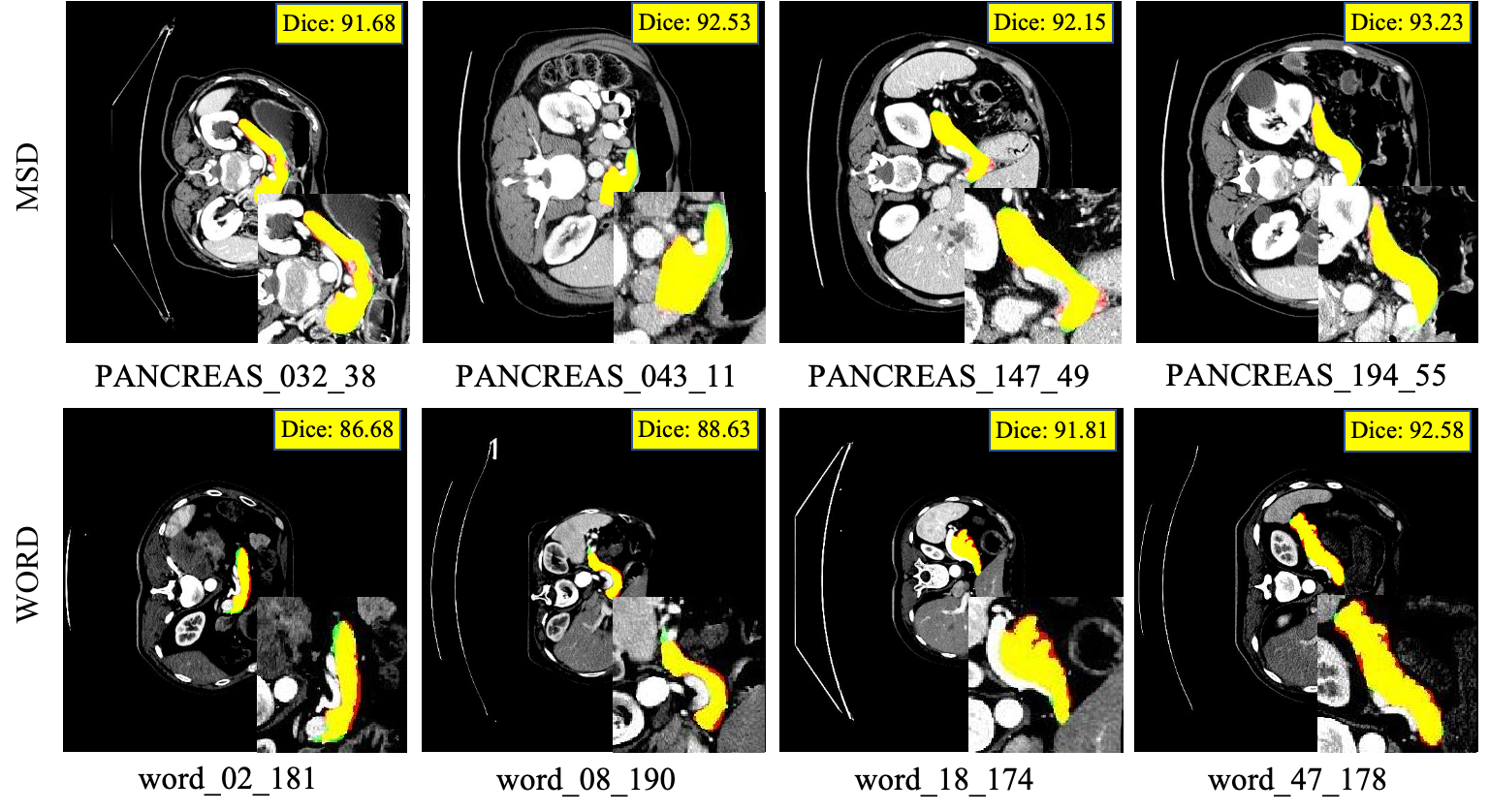}
  \caption{Segmentation results on the MSD dataset (first row) and the WORD dataset (second row). Contours: Red = Gold Standard, Green = Predicted Result.}
  \label{msdword}
\end{figure}
\begin{equation}
\begin{aligned}
\mu_i(\phi) &=\frac{\int_{\Omega} I\,\chi_i(\phi)\,\mathrm{d}y}{\int_{\Omega} \chi_i(\phi)\,\mathrm{d}y},\\
\Sigma_i(\phi) &=\frac{\int_{\Omega} (I-\mu_i)(I-\mu_i)^\top\,\chi_i(\phi)\,\mathrm{d}y}{\int_{\Omega} \chi_i(\phi)\,\mathrm{d}y}.
\end{aligned}
\end{equation}
This proposition extends to our generalized region term \(\mathcal{E}_{\mathrm{Region}}\), which is based on regional Gaussian statistics.

\begin{proposition}
\label{prop:region_drive}
\textbf{(Boundary consistency of two--phase regional energy descent and topological drive.)}
The boundary drive of the generalized region term \(\mathcal{E}_{\mathrm{Region}}\) is proportional to the difference of the negative log--likelihoods at the boundary:
\begin{equation}
\label{eq:region-boundary-drive}
\Big\langle -\tfrac{\delta \mathcal{E}_{\mathrm{Region}}}{\delta \phi},\,n\Big\rangle\Big|_{\Gamma}
\;\propto\; -e_1(x) + e_2(x).
\end{equation}
\end{proposition}
This expression can be interpreted as a two--phase generalization of the classic CV model's contrast term. As proven in Appendix~\ref{app:region_proof}, this boundary drive is also consistent with the direction given by the term's topological derivative.

\subsubsection{Length Term}
To suppress boundary oscillations and enforce smoothness, we use the perimeter surrogate:
\begin{equation}
\mathcal{E}_{\mathrm{Length}}(\Gamma) = \mathrm{Length}(\Gamma) = \int_{\Gamma} \mathrm{d}s.
\end{equation}
Here, $\mathrm{d}s$ represents the arc length element along the boundary. This boundary integral is transformed into an integral over the entire domain $\Omega$ using its level set functional form:
\begin{equation}
\label{eq:E-length}
\mathcal{E}_{\mathrm{Length}}(\phi)=\int_{\Omega}\big|\nabla H(\phi)\big|\,\mathrm{d}y.
\end{equation}

\subsubsection{Area Term}
The pancreas occupies a small volume; we thus penalize deviations from target areas to stabilize scale:
\begin{equation}
\label{area}
\mathcal{E}_{\mathrm{Area}}(\Omega_1, \Omega_2) = (\int_{\Omega_1}\mathrm{d}y - A_1)^2 + (\int_{\Omega_2}\mathrm{d}y - A_2)^2, 
\end{equation}
where \(A_1\) is the target area, \(A_2=\Omega \setminus A_1\). $\int_{\Omega_i}\mathrm{d}y$ represents the areas of the foreground and the background respectively. The corresponding level set form, derived using $H(\phi)$, is: 
\begin{equation}
\label{eq:E-area}
\mathcal{E}_{\mathrm{Area}}(\phi)
=\Big(\!\int_{\Omega} H(\phi)\,\mathrm{d}y-A_1\Big)^2
+\Big(\!\int_{\Omega} \big(1-H(\phi)\big)\,\mathrm{d}y-A_2\Big)^2.
\end{equation}
We treat this as a coarse size prior.

\begin{algorithm*}[t]
\caption{Training Algorithm for TA-LSDiff}
\label{alg:training}
\begin{algorithmic}[1]
\REQUIRE
    Training dataset \(\mathcal{D} = \{(I_i, y_{0,i})\}_{i=1}^N\); Denoising network \(\epsilon_\theta\); Loss weights \(\eta_1, \eta_2\); Total diffusion timesteps \(T\); Pixel--adaptive Refinement (PAR) iterations \(\tau\).

\STATE Initialize network parameters \(\theta\).
\REPEAT
    \STATE Sample \((I, y_0) \sim \mathcal{D}\); \(t \sim \mathrm{Uniform}(\{1, \dots, T\})\); \(\epsilon \sim \mathcal{N}(0, \mathbf{I})\).
    
    \COMMENT{--- Diffusion Denoising Step ---}
    \STATE Construct noisy mask: \(y_t \leftarrow \sqrt{\bar{\alpha}_t}y_0 + \sqrt{1-\bar{\alpha}_t}\epsilon\). 
    \STATE Compute the diffusion loss: \(L_{\mathrm{dpm}} \leftarrow ||\epsilon - \epsilon_\theta(y_t, I, t)||_2^2\). 

    \COMMENT{--- Level Set Step ---}
    \STATE Predict the clean mask from the noise prediction: \(\hat{y}_0 \leftarrow (y_t - \sqrt{1-\bar{\alpha}_t}\epsilon_\theta(y_t, I, t)) / \sqrt{\bar{\alpha}_t}\).
    
    \STATE Let the level set function be \(\phi \leftarrow \hat{y}_0\).
    \STATE Compute the level set energy loss: \(L_{\mathrm{lsf}} \leftarrow \mathcal{E}(\phi; I)\). 

    \STATE Initialize PAR input: \(y_{\phi,0} \leftarrow \hat{y}_\phi\).
    \STATE Compute affinity kernel \(\kappa\) from image \(I\). 
    \STATE Refine the mask: \(y_{\phi,\tau}^{i,j} = \sum_{(k,l)\in\mathcal{N}_{8}(i,j)}\kappa^{ij,kl}y_{\phi,\tau-1}^{k,l}\). 
    \STATE Compute the PAR loss: \(L_{\mathrm{par}} \leftarrow ||\hat{y}_0 - y_{\phi,\tau}||_1\). 
    
    \COMMENT{--- Gradient Update Step ---}
    \STATE Compute total loss: \(\mathcal{L} \leftarrow L_{\mathrm{dpm}} + \eta_1 L_{\mathrm{lsf}} + \eta_2 L_{\mathrm{par}}\). 
    \STATE Update network parameters \(\theta\) using a gradient step on \(\nabla_\theta \mathcal{L}\).
\UNTIL{convergence}
\end{algorithmic}
\end{algorithm*}

\subsubsection{Distance Penalty Term}
We incorporate a geodesic distance to aid in localization and suppress the background:
\begin{equation}
\mathcal{E}_{\mathrm{Distance}}(\Omega_1)=\int_{\Omega_1} D(y)\mathrm{d}y.
\end{equation}
Similarly,
\begin{equation}
\label{eq:E-dist}
\mathcal{E}_{\mathrm{Distance}}(\phi;I)=\int_{\Omega} D(y)\,H(\phi(y))\,\mathrm{d}y,
\end{equation}
where
\begin{equation}
\begin{split}
\begin{dcases}
\left| \nabla \mathcal{D}_{\Omega_1}^0(y) \right| = f(y) \quad &y \notin \Omega_1\\
\mathcal{D}_{\Omega_1}^0(y) = 0 \quad \quad \quad & y \in \Omega_1,
\end{dcases}
\end{split}
\end{equation}
and we define the normalized distance \(\mathcal{D}_{\Omega_1}(y)=\mathcal{D}_{\Omega_1}^0(y)/\|\mathcal{D}_{\Omega_1}^0\|_{L^\infty}\).
We take
\begin{equation}
f(y)=\varepsilon_{\mathcal{D}}+\beta_G\|\nabla I(y)\|^2+\nu\,\mathcal{D}_E(y).
\end{equation}
This ensures that homogeneous regions are "cheap" (flat distance) while edges are "costly" (large gradients)21. In practice, we set \(\beta_G=10^3,\ \varepsilon_{\mathcal{D}}=10^{-3}\).

\subsubsection{Overall Level Set Energy and Variation}
Our deep level set loss is defined in \eqref{total-variational}. Its gradient--flow evolution (shape derivative) is given by:
\begin{equation}
\label{eq:phi-evolution}
\begin{aligned}
\frac{\partial \phi}{\partial t}
& =-\,\frac{\partial L_{\mathrm{lsf}}}{\partial \phi} \\
& =\delta_\varepsilon(\phi) \Bigg[
\theta_1\,(e_2-e_1)
+\theta_2\,\mathrm{div}\!\Big(\frac{\nabla\phi}{\|\nabla\phi\|}\Big)-\theta_4\,D(y) \\
& \quad -\theta_3 \Big(
\small(\!\int H(\phi)\,\mathrm{d}y-A_1\small)
-\small(\!\int H(\phi)\,\mathrm{d}y-A_2\small) \Big) \Bigg].
\end{aligned}
\end{equation}
This can be discretized as \(\phi^{n+1}=\phi^{n}+\Delta t\,\partial_t\phi^{n}\). However, in our framework, we do not explicitly evolve $\phi$. Instead, we backpropagate through the mask $y_\phi$.

\subsubsection{Pixel--Adaptive Refinement}
\label{sec:par}
The segmentation results often contain local inconsistencies (e.g.,neighboring pixels with similar low--level appearance but different semantics). We therefore introduce a pixel--adaptive refinement (PAR) that updates \(\phi\) by convex combinations of local neighbors.
Given a pixel feature \(p_{ij}\), the pairwise affinity over the 8--neighborhood is defined as:

\begin{equation}
\begin{aligned}
\kappa^{ij,kl}_p &=-\Big(\frac{|p^{ij}-p^{kl}|}{\sigma^{ij}_p}\Big)^2,\\
\kappa^{ij,kl} &=\frac{\exp(\kappa^{ij,kl}_p)}{\sum_{(k,l)\in\mathcal{N}_8(i,j)}\exp(\kappa^{ij,kl}_p)}.
\end{aligned}
\end{equation}
We perform \(\tau\) iterations of gradual optimization on the \(y_\phi\):

\begin{equation}
y^{i,j}_{\phi,\tau}=\sum_{(k,l)\in\mathcal{N}_8(i,j)} \kappa^{ij,kl}\,y^{i,j}_{\phi,\tau-1},
\end{equation}
and enforce consistency via an \(\ell_1\) penalty between the \(y_{\phi}\) and the refined \(y_{\phi;\tau}\):
\begin{equation}
\label{eq:Lpar}
L_{\mathrm{par}}=\|y_\phi-y_{\phi,\tau}\|_1.
\end{equation}
We use $\tau=10$ iterations by default to improve robustness.

\begin{figure}
  \centering
  \hspace*{-0.7cm}\includegraphics[width=3.8in, trim=-0.4in 0 0 0, clip]{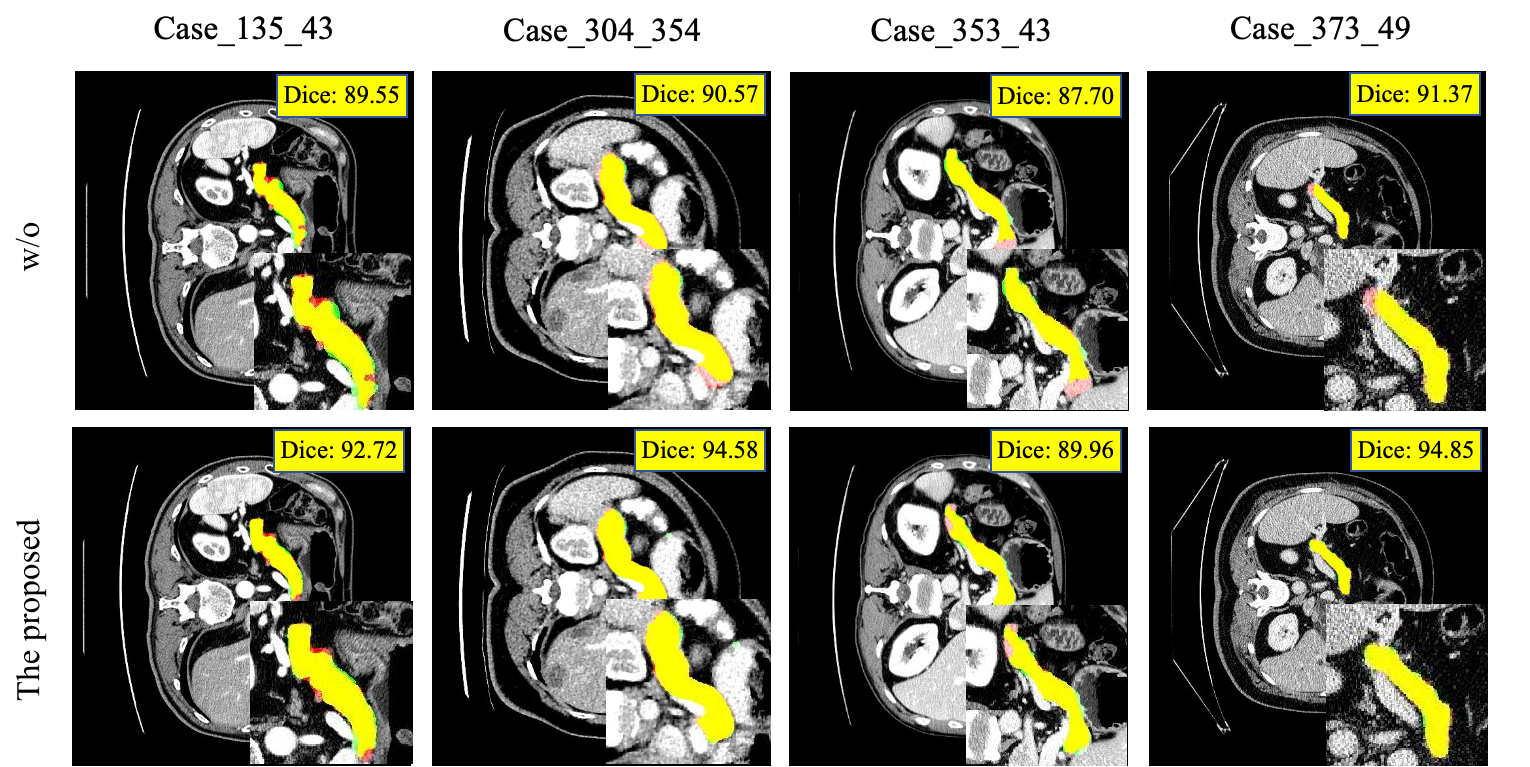}
  \caption{Comparison of segmentation results on different instances without (first row) and with (second row) the level set evolution. Contours: Red = Gold Standard, Green = Predicted Result.}
  \label{fig-duibi}
\end{figure}
\begin{figure}
  \centering
  \hspace*{-0.6cm}\includegraphics[width=3.7in, trim=-0.4in 0 0 0, clip]{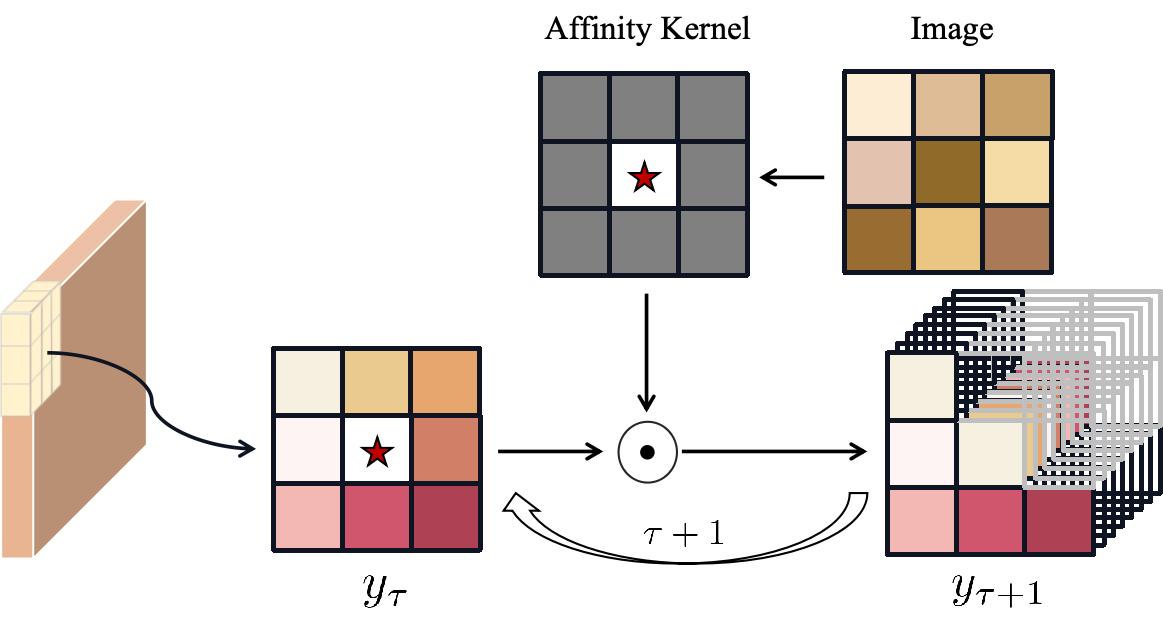}
  \caption{Conceptual illustration of Pixel--adaptive Refinement (PAR). An affinity kernel is computed for each pixel to measure its proximity to its neighbors. This kernel is then iteratively applied to the predicted mask (not the level set evolution) via adaptive convolution to obtain the refined segmentation result.}
  \label{par}
\end{figure}
\subsection{Training Loss and Inference}
\label{sec:loss-infer}

We optimize the network end-to-end with three components: a diffusion loss, the level set energy, and pixel--adaptive refinement (PAR):
\begin{equation}
\label{eq:total-loss}
\mathcal{L}=\;L_{\mathrm{dpm}}\;+\;\eta_1\,L_{\mathrm{lsf}}\;+\;\eta_2\,L_{\mathrm{par}},
\end{equation}
where $eta_i$ are weights of training loss and each term is presented in \eqref{eq:ldpm-eps}, \eqref{total-variational}, and \eqref{eq:Lpar}.

\subsubsection{DDPM Training Objective}
The network is conditioned on the image \(I\). We train it using the standard $\varepsilon$-prediction form:
\begin{equation}
\label{eq:ldpm-eps}
L_{\mathrm{dpm}} = \mathbb{E}_{t, y_0, \epsilon} \left[ w_t || \epsilon - \epsilon_{\theta}(y_t, I, t) ||_2^2 \right],
\end{equation}
with an optional weighting term \(w_t\).

\subsubsection{Network Architecture}In the diffusion probabilistic model of the segmentation process, we provide an additional input of the original image $I$ to guide the generation of segmentation results by the diffusion probabilistic model. $\varepsilon_{\theta}$ is typically a U-Net represented as follows:\begin{equation}\varepsilon_{\theta}(y_t, I, t) = D\left(E_A\left(E_B(y_t,t) + E_C(I,t), t\right), t\right).\end{equation}In this architecture, the decoder $D$ of the U-Net is conventional, while its encoder is decomposed into three networks: $E_A$, $E_B$, and $E_C$. $E_C$ represents conditional feature embedding, which embeds the original image; $E_B$ represents feature embedding of the segmentation map for the current step. The encoders consist of three convolutional stages. The residual blocks for each stage follow the structure of ResNet34, comprising two convolutional blocks with group--norm and SiLU \cite{elfwing2018sigmoid} active layer, as well as a convolutional layer. These two processed inputs have the same spatial dimensions and channel numbers and are summed up as signals. The summation is then passed onto the remaining part of U-Net's encoder $E_A$ and sent to U-Net's decoder $D$ for reconstruction. The time step $t$ is integrated with embeddings.
\begin{table}[t!]
\caption{Comparison on five--fold cross--validation}
\label{five--fold}
\centering
\newcolumntype{C}[1]{>{\centering\arraybackslash}m{#1}}
{\scriptsize
\setlength{\tabcolsep}{2pt}
\renewcommand{\arraystretch}{1.4}
\begin{tabular}{C{45pt}|C{45pt}|C{45pt}|C{45pt}|C{45pt}}
\hline
Metrics & Dice & Jaccard & Precision & Recall \\
\hline
& Average & Average & Average & Average \\[-4pt]
& (max, min) & (max, min) & (max, min) & (max, min) \\
\hline
First Fold &
91.05$\pm$4.55 (98.00, 76.92) &
83.88$\pm$7.32 (96.07, 62.49) &
94.60$\pm$4.73 (99.96, 64.52) &
88.20$\pm$7.13 (98.83, 62.85) \\
Second Fold &
90.53$\pm$4.94 (97.79, 77.16) &
83.06$\pm$7.95 (95.67, 62.81) &
93.63$\pm$5.19 (100.00, 74.95) &
88.33$\pm$8.59 (99.85, 63.62) \\
Third Fold &
92.12$\pm$4.38 (97.94, 77.16) &
85.69$\pm$7.32 (95.95, 62.81) &
93.23$\pm$5.56 (100.00, 68.78) &
91.52$\pm$6.57 (99.65, 63.80) \\
Fourth Fold &
91.05$\pm$4.98 (97.67, 77.09) &
83.93$\pm$8.08 (94.43, 62.73) &
93.67$\pm$5.55 (100.00, 66.25) &
89.26$\pm$8.44 (99.88, 64.27) \\
Fifth Fold &
90.64$\pm$5.02 (97.93, 76.85) &
83.26$\pm$8.09 (95.94, 62.40) &
94.20$\pm$5.71 (100.00, 66.41) &
88.05$\pm$8.38 (100.00, 63.61) \\
Five--fold cross &
91.09$\pm$4.80 (98.00, 76.85) &
83.98$\pm$7.80 (96.07, 62.40) &
93.87$\pm$5.37 (100.00, 64.52) &
89.09$\pm$7.94 (100.00, 62.85) \\
\hline
\end{tabular}}
\end{table}
\begin{figure}
  \centering
  \hspace*{-0.8cm}\includegraphics[width=3.8in, trim=-0.4in 0 0 0, clip]{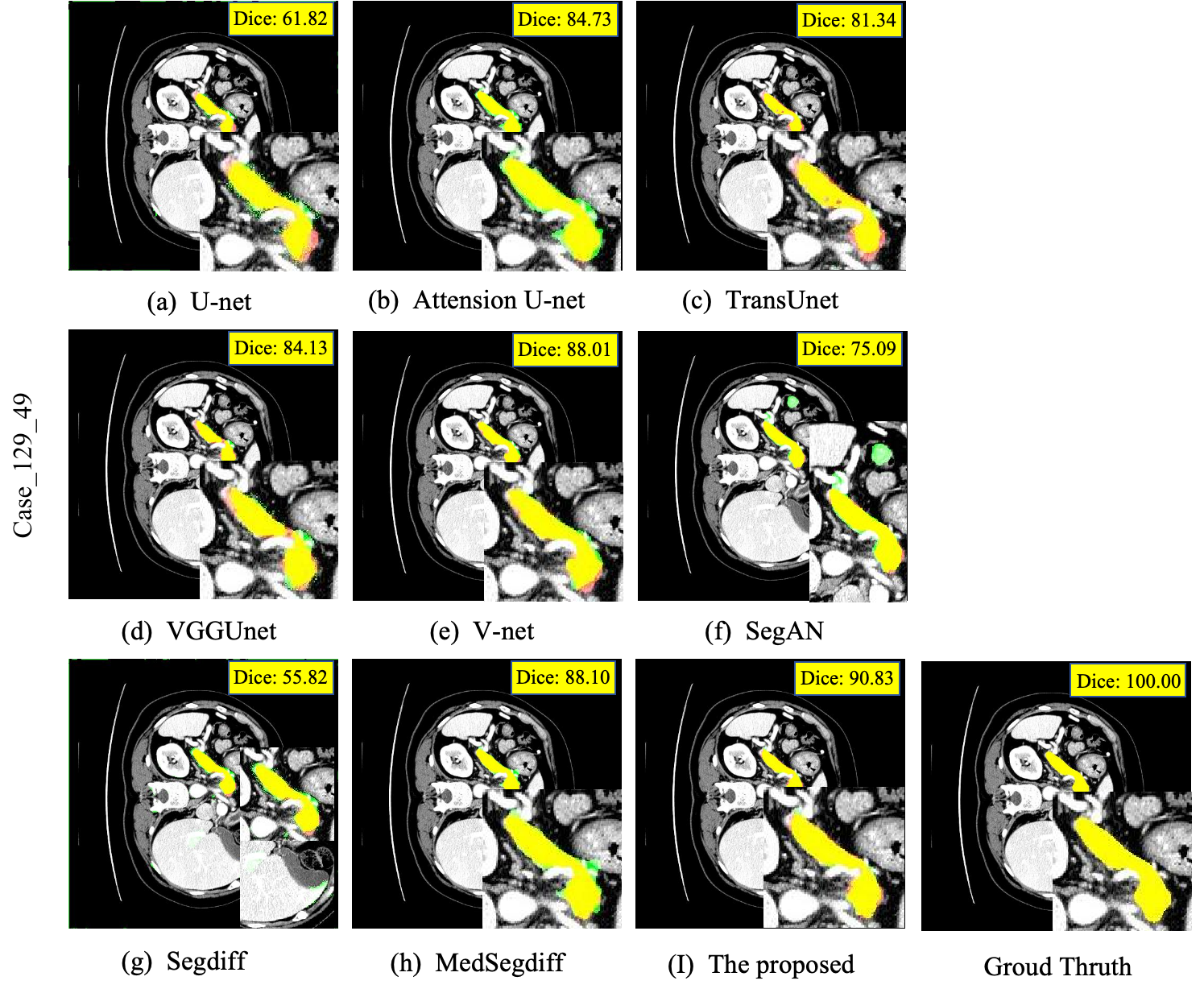}
  \caption{Comparison of segmentation results from different methods on the same instance. Contours: Red = Gold Standard, Green = Predicted Result.}
  \label{gezhong}
\end{figure}
\subsubsection{Inference and Energy Regularized Sampling.}
At inference, the model directly outputs a mask without requiring explicit geometric evolution. We inject the energy descent at each reverse step
\begin{equation}
\nabla_{y_t} \left( \log p(y_t \mid I) - \gamma_t L_{\mathrm{lsf}} \right) = \nabla_{y_t} \log p(y_t \mid I) - \gamma_t \nabla_{y_t} L_{\mathrm{lsf}}.
\end{equation}
This process is equivalent to ascending a regularized posterior $\log p(y_t\!\mid\!I)-\gamma_t L_{\mathrm{lsf}}$. This yields topology--aware, smooth, and well--connected boundaries.

\section{Experiments}
\subsection{Implemention Detail} %

This section outlines the datasets and evaluation metrics, data preprocessing methods, dataset partitioning strategy, and the experimental configurations used for model training and evaluation.

\subsubsection{Datasets} 

We evaluate on four public abdominal CT datasets with pancreas labels. AbdomenCT-1K \cite{ma2021abdomenct}: 1{,}112 scans (1{,}000 publicly released with labels), multi--organ; we extract the pancreas labels. NIH \cite{roth2015deeporgan}: 82 contrast--enhanced CT volumes with expert--verified pancreas annotations. MSD \cite{antonelli2022medical}: 281 training and 139 test cases from MSKCC with pancreas labels. WORD \cite{luo2022word}: 150 CTs (~30{,}495 slices) with pixel--level annotations of 16 organs, we extract the pancreas labels. 

\subsubsection{Evaluation Metrics}
\label{sec:metrics}
We evaluate foreground segmentation with Dice, Jaccard, Precision, and Recall. These metrics are computed using pixel--level counts of True Positives (\(TP\)), False Positives (\(FP\)), and False Negatives (\(FN\)):
\[\begin{aligned}
\text{Dice}&=\frac{2TP}{2TP+FP+FN},\quad \text{Jaccard}=\frac{TP}{TP+FP+FN},\quad \\
\text{P}&\text{recision}=\frac{TP}{TP+FP},\quad
\text{Recall}=\frac{TP}{TP+FN}.
\end{aligned}\]
All metrics range in \([0,1]\) (higher is better). 

\subsubsection{Experimental Setting}

We train our model on an RTX 3090 GPU with an initial learning rate of 0.0001, diffusion steps of 1,000 iterations, and a batch size of 2 samples. 

The weights for the loss components were determined empirically through experiments. For the main loss terms \eqref{eq:total-loss}, we set $\eta_1=0.5$ and $\eta_2 =0.005$. For the level set energy components \eqref{total-variational}, the weights were set to $\lambda_1 = 0.01$, $\lambda_2 = 0.01$, $\lambda_3 = 0.0001$, and $\lambda_4= 0.001$.

In the training process, the model is trained for a total of 400,000 to 450,000 iterations. In the testing process, instead of using one single prediction result of the model, we take the average of the 20 prediction results as the final prediction output for a single input.

\subsubsection{Datasets Setting}

For image preprocessing, all input 2D CT image slices have a size of 512×512. Pancreas densities are known to typically fall within the 40--70 Hounsfield Unit (HU) range. Therefore, by setting the window width to 250 and window level to 50, we effectively adjust the intensity values of all pixels to [-75,175] HU to better visualize the pancreas. 

In the five--fold cross--validation experiments conducted on the AbdomenCT-1K dataset, we extracted 2,000 2D slices from the first 400 volumetric cases to constitute the dataset.  These 2,000 2D slices were equally partitioned into five folds, each containing 400 slices.  During the five--fold cross--validation experiments, each fold sequentially served as the test dataset while the remaining four folds were combined for training, with the final performance metric derived from averaging results across all five folds.  

For subsequent comparative analyses, one fold of the AbdomenCT-1K dataset containing 1,600 2D training slices and 400 2D testing slices was selected as the benchmark configuration.  For the NIH, MSD, and WORD datasets, a total of 1,800 clinically validated 2D slices were curated, with 1,500 slices allocated for model training and 300 slices reserved for testing.  

\begin{table}[t!]
\caption{Comparison on the exact same dataset}
\label{exact}
\centering
\newcolumntype{C}[1]{>{\centering\arraybackslash}m{#1}} 
\setlength{\tabcolsep}{2pt}
\renewcommand{\arraystretch}{2.3}
{\scriptsize
\begin{tabular}{C{60pt}|C{50pt}|C{40pt}|C{40pt}|C{40pt}}
\hline
Methods & AbdomenCT-1K & NIH & MSD & WORD \\
\hline
Metrics & \makecell{Dice\\Jaccard} & \makecell{Dice\\Jaccard} & \makecell{Dice\\Jaccard} & \makecell{Dice\\Jaccard} \\
\hline
Segdiff \cite{amit2021segdiff} &
\makecell{60.90$\pm$12.28\\44.97$\pm$13.64} &
\makecell{59.69$\pm$11.98\\45.01$\pm$13.90} &
\makecell{59.63$\pm$12.81\\44.18$\pm$14.68} &
\makecell{59.17$\pm$12.57\\45.36$\pm$13.96} \\
MedSegdiff \cite{wu2022medsegdiff} &
\makecell{88.04$\pm$8.02\\79.46$\pm$11.52} &
\makecell{86.08$\pm$4.52\\75.84$\pm$6.81} &
\makecell{85.11$\pm$7.06\\74.72$\pm$10.50} &
\makecell{83.53$\pm$7.96\\71.82$\pm$9.51} \\
SegAN \cite{xue2018segan} &
\makecell{80.98$\pm$7.02\\68.61$\pm$9.77} &
\makecell{78.67$\pm$7.82\\68.18$\pm$8.99} &
\makecell{77.32$\pm$9.23\\68.09$\pm$11.01} &
\makecell{76.81$\pm$7.31\\67.19$\pm$9.71} \\
U\text{-}net \cite{bateriwala2019enforcing} &
\makecell{77.75$\pm$10.28\\64.75$\pm$13.90} &
\makecell{74.49$\pm$8.76\\62.67$\pm$11.22} &
\makecell{72.90$\pm$10.98\\61.71$\pm$12.77} &
\makecell{72.57$\pm$8.72\\62.56$\pm$10.01} \\
Attention U\text{-}net \cite{oktay2018attention} &
\makecell{80.49$\pm$18.60\\70.78$\pm$22.45} &
\makecell{78.81$\pm$12.66\\69.24$\pm$15.10} &
\makecell{76.38$\pm$17.81\\65.67$\pm$19.92} &
\makecell{76.31$\pm$13.31\\65.71$\pm$15.73} \\
V\text{-}net \cite{milletari2016v} &
\makecell{88.64$\pm$12.32\\81.32$\pm$15.78} &
\makecell{85.36$\pm$9.76\\77.31$\pm$12.12} &
\makecell{84.43$\pm$9.39\\76.17$\pm$11.36} &
\makecell{83.47$\pm$8.90\\73.31$\pm$11.91} \\
TransUnet \cite{chen2021transunet} &
\makecell{79.91$\pm$9.13\\70.31$\pm$10.98} &
\makecell{75.97$\pm$9.21\\67.59$\pm$11.31} &
\makecell{73.76$\pm$10.61\\65.52$\pm$12.00} &
\makecell{73.47$\pm$9.89\\65.36$\pm$11.47} \\
VGGUnet \cite{liu2022pancreas} &
\makecell{81.91$\pm$8.31\\71.14$\pm$9.87} &
\makecell{79.10$\pm$7.36\\70.09$\pm$9.31} &
\makecell{77.31$\pm$9.87\\69.71$\pm$10.97} &
\makecell{76.89$\pm$8.64\\67.62$\pm$9.94} \\
\textbf{The proposed} &
\makecell{\textbf{91.05$\pm$4.55}\\\textbf{83.88$\pm$7.32}} &
\makecell{\textbf{87.93$\pm$3.27}\\\textbf{78.56$\pm$5.11}} &
\makecell{\textbf{86.37$\pm$6.46}\\\textbf{76.37$\pm$9.88}} &
\makecell{\textbf{85.36$\pm$7.63}\\\textbf{73.84$\pm$8.83}} \\
\hline
\end{tabular}}
\end{table}

\subsection{Experimental Results and Discussion} %

This section presents comprehensive analyses of the proposed framework's performance under identical experimental configurations across multiple datasets, demonstrating significant metric improvements over state-of-the-art methods while exhibiting robust generalization capabilities. 

\subsubsection{Five--Fold Cross--Validation}

This section provides a comprehensive evaluation of the model generalization by performing five--fold cross--validation on the AbdomenCT-1K dataset. As shown in TABlE \ref{five--fold}, a quantitative comparison of all validation sections shows robust performance stability, with the proposed model achieving state-of-the-art performance across all key evaluation metrics including mean, standard deviation, maximum, and minimum. In addition, we can see that in addition to the comprehensive five--fold cross--validation results, we show the five--fold separate results, and we can see that the performance of each fold is advanced and stable. As shown in Fig.\ref{fig-duibi}, we can see that the diffusion probability model combined with the level set model significantly improves the performance of the model. Both over--segmentation and under--segmentation of the top of the pancreas of $Case\_135\_43$ are significantly improved after level set evolution. In the last three groups of instances, there are obvious undersegmentation in the results of the top and bottom of the pancreas in the first row. After the level set evolution, the corresponding results are successfully segmented. Our model can effectively compensate and improve the undersegmentation problem of the diffusion probability model, and alleviate the over--segmentation problem.

\begin{figure*}
  \centering
  \hspace*{-1cm}\includegraphics[width=6.5in, trim=-0.4in 0 0 0, clip]{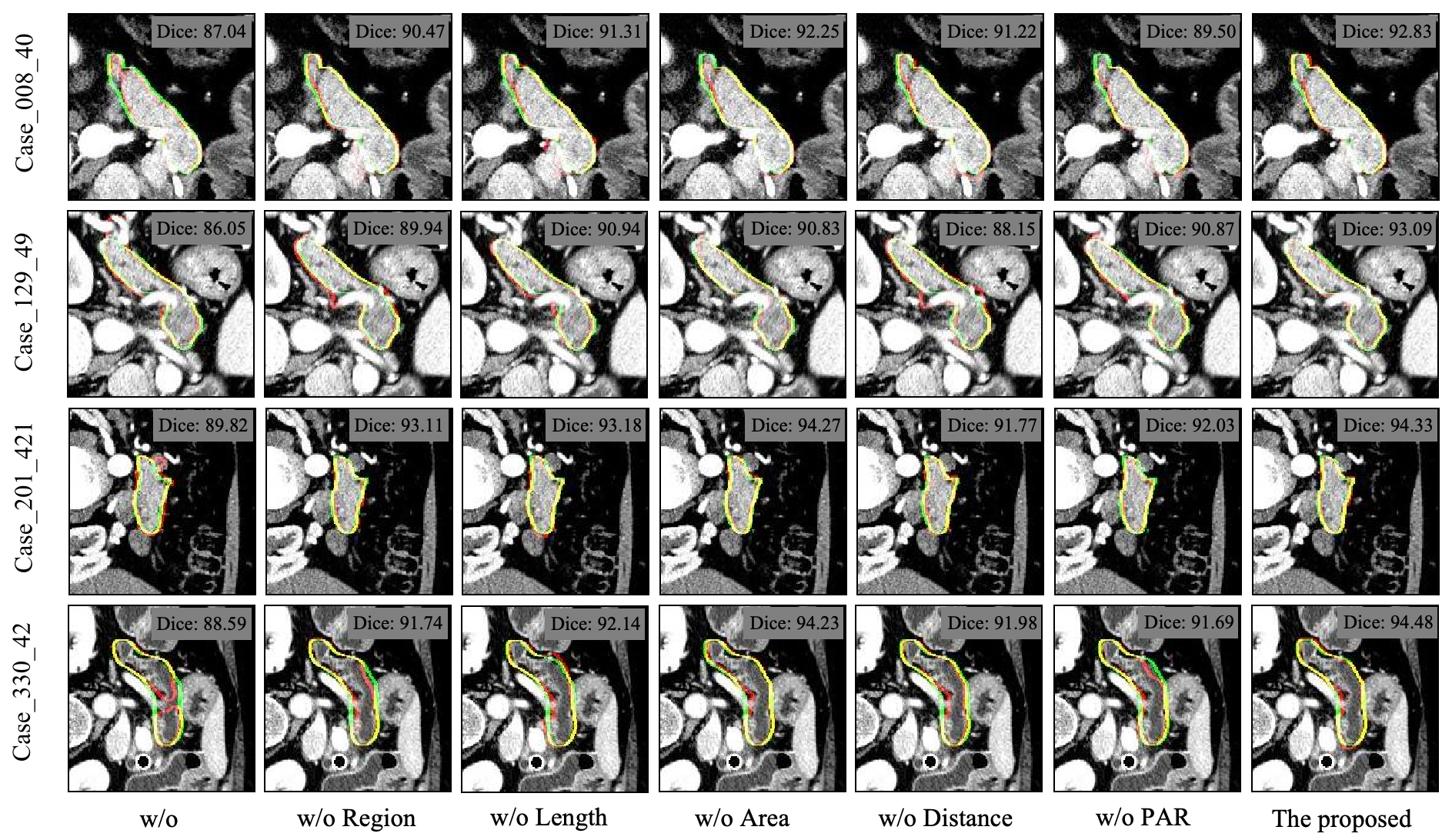}
  \caption{Comparison of segmentation results from different level set components on the same instance. Contours: Green = Gold Standard, Red = Predicted Result.}
  \label{fig:xiaorong}
\end{figure*}
\begin{table}
\caption{Comparison of other metrics}
\label{other}
\centering
\newcolumntype{C}[1]{>{\centering\arraybackslash}m{#1}}
\setlength{\tabcolsep}{3pt}
\renewcommand{\arraystretch}{1.5}
{\scriptsize
\begin{tabular}{C{27pt}|C{56pt}|C{33pt}|C{35pt}|C{35pt}|C{35pt}}
\hline
Dataset & Method & Dice & Jaccard & Precision & Recall \\
\hline
Abdoment & PanSegNet\cite{zhang2025large} & 88.31$\pm$7.24 & 79.71$\pm$10.12 & 87.77$\pm$9.29 & \textbf{90.08$\pm$8.86} \\
CT-1K  & The proposed & \textbf{91.09$\pm$4.80} & \textbf{83.98$\pm$7.80} & \textbf{93.87$\pm$5.37} & 89.09$\pm$7.94 \\
\hline
 & DAN\cite{li2021automatic} & 82.38$\pm$5.46 & 70.39$\pm$7.58 & 83.51$\pm$7.63 & 82.30$\pm$7.45 \\
 & RNN\cite{cai2019pancreas} & 83.30$\pm$5.60 & 71.80$\pm$7.70 & 84.50$\pm$6.20 & 82.80$\pm$8.37 \\
 & TEFCN\cite{liu2019automatic} & 84.10$\pm$4.91 & 72.86$\pm$6.89 & 83.60$\pm$5.85 & 85.33$\pm$8.24 \\
 & TVMS-Net\cite{chen2022pancreas} & 85.19$\pm$4.73 & 74.19$\pm$7.27 & 86.09$\pm$5.93 & 84.58$\pm$8.09 \\
 & FPF-Net\cite{chen2023fpf} & 85.41$\pm$4.47 & 74.80$\pm$6.30 & 85.60$\pm$5.90 & 85.90$\pm$6.50 \\
 & LocNet-ECTN\cite{zheng2023extension} & 85.58$\pm$3.98 & 74.99$\pm$5.86 & 86.59$\pm$6.14 & 85.11$\pm$5.96 \\
NIH & MDS-Net\cite{li2020model} & 85.70$\pm$4.10 & 75.30$\pm$6.10 & 87.40$\pm$5.20 & 84.80$\pm$7.50 \\
 & CM3D-FCN\cite{xue2019cascaded} & 85.90$\pm$5.10 & 75.70$\pm$7.60 & 87.60$\pm$4.70 & 85.20$\pm$8.90 \\
 & TPA\cite{dogan2021two} & 86.15$\pm$4.45 & 75.93$\pm$6.46 & 86.23$\pm$4.85 & 86.27$\pm$5.73 \\
 & tUNet\cite{chen2022target} & 86.38$\pm$3.18 & 76.16$\pm$4.79 & 92.00$\pm$3.47 & 81.62$\pm$4.91 \\
 & CTUNet\cite{chen2023ctunet} & 86.80$\pm$4.10 & 76.90$\pm$6.10 & 86.20$\pm$6.50 & 88.00$\pm$6.00 \\
 & MR-Net\cite{li2020multiscale} & 87.57$\pm$3.26 & \textbf{78.77$\pm$4.34} & 86.63$\pm$3.70 & \textbf{89.55$\pm$4.03} \\
 & tU-Net+\cite{chen2022target} & 87.91$\pm$2.65 & 78.52$\pm$4.14 & 90.43$\pm$3.77 & 85.77$\pm$4.61 \\
 & The proposed & \textbf{87.93$\pm$3.27} & 78.56$\pm$5.11 & \textbf{92.06$\pm$4.21} & 84.45$\pm$6.20 \\
\hline
 & CM3D-FCN\cite{xue2019cascaded} & 73.60$\pm$9.70 & 59.10$\pm$11.80 & 84.30$\pm$10.40 & 67.20$\pm$13.70 \\
MSD & TVMS-Net\cite{chen2022pancreas} & 76.60$\pm$7.30 & 62.60$\pm$9.30 & 87.70$\pm$8.30 & 69.20$\pm$12.80 \\
 & The proposed & \textbf{86.37$\pm$6.46} & \textbf{76.57$\pm$9.88} & \textbf{94.29$\pm$5.32} & \textbf{80.56$\pm$6.88} \\
\hline
WORD  & PanSegNet\cite{zhang2025large} & 80.89$\pm$7.48 & 68.51$\pm$9.60 & 85.47$\pm$12.46 & 78.17$\pm$6.77 \\
 & The proposed & \textbf{85.36$\pm$7.63} & \textbf{73.84$\pm$8.83} & \textbf{92.46$\pm$8.53} & \textbf{84.74$\pm$7.27} \\
\hline
\end{tabular}}
\end{table}
\subsubsection{Benchmark Performance on the Standard Test Set}
This section presents comparative experimental results and analyses under identical training and testing configurations. As shown in TABlE \ref{exact}, we evaluate our proposed framework against classical and state-of-the-art pancreatic segmentation models, including conventional U-Net, attention--enhanced architectures, Transformer--based variants, adversarial approach and diffusion probabilistic model. The U-Net demonstrates fundamental localization capability but suffers from severe under--segmentation in anatomically complex abdominal CT images. Integration of attention mechanisms with Transformer and VGG backbones yields measurable precision improvements, where the V-net variant with asymmetric layer configurations and multi--scale channels exhibits enhanced feature extraction capacity. While the GAN--based SegAN achieves plausible organ localization, it introduces significant background artifacts. The diffusion--based SegDiff model, lacking original image guidance, fails to accurately segment small target organs amidst complex backgrounds. The MedSegDiff incorporats attention--guided original image, shows enhanced performance but generates erroneous background speckles. Our proposed framework addresses these limitations through novel integration of deep level set loss and pixel--adaptive refinement module, achieving superior segmentation accuracy with 3.01\% Dice improvement over MedSegDiff.

As shown in Fig.\ref{gezhong}, we compare the results of applying different methods on the same instance. We can see that all the above deep learning methods can achieve ROI localization, where the results of applying the generative model (f--h) have different levels of background interference. The U-net based methods (a--e) will have different degrees of error, and the proposed model shows significantly superior performance.

\subsubsection{Multi--Dataset Robustness Evaluation}

Similar results could also be obtained on other datasets, and the proposed model demonstrated accurate and robust performance. Compared with the existing state-of-the-art methods shown in TABlE \ref{other}, the mean and variance of the proposed method are very stable. The proposed model achieved the optimal results on the main evaluation metric Dice, and the optimal and suboptimal results on the Jaccard metric. While maintaining the coverage rate, the other two auxiliary indicators also demonstrated excellent performance. As shown in Fig.\ref{abknih} and Fig.\ref{msdword}, most of our segmentation results overlap with the gold standard on four datasets. 

\subsection{Ablation Experiments} %

This section systematically investigates the contribution of individual components through controlled ablation studies employing sequential module removal, while rigorously validating the model's operational reliability and generalization capacity via statistical significance tests and cross--dataset evaluations.  

\begin{table}
\caption{Comparison of other metrics}
\label{tab:xiaorong}
\centering
\newcolumntype{C}[1]{>{\centering\arraybackslash}m{#1}}
\setlength{\tabcolsep}{3pt}
\renewcommand{\arraystretch}{1.5}
{\scriptsize
\begin{tabular}{C{32pt}|C{32pt}|C{32pt}|C{32pt}|C{32pt}|C{52pt}}
\hline
Region & Length & Area & Distance & PAR & Dice \\
\hline
 &  &  &  &  & 88.04$\pm$8.02 \\
 & \checkmark & \checkmark & \checkmark & \checkmark & 89.25$\pm$7.78 \\
\checkmark &  & \checkmark & \checkmark & \checkmark & 90.50$\pm$8.95 \\
\checkmark & \checkmark &  & \checkmark & \checkmark & 90.02$\pm$8.15 \\
\checkmark & \checkmark & \checkmark &  & \checkmark & 89.67$\pm$9.32 \\
\checkmark & \checkmark & \checkmark & \checkmark &  & 89.29$\pm$11.88 \\
\checkmark & \checkmark & \checkmark & \checkmark & \checkmark & 91.05$\pm$4.55 \\
\hline
\end{tabular}}
\end{table}
This section investigates the impact of different components in the deep level set loss and the pixel--adaptive refinement module through systematic ablation studies, as demonstrated in TABlE \ref{tab:xiaorong}.  The checkmarks in the table indicate the inclusion of specific components in each experimental configuration.  Comparative analysis reveals that the region fitting term plays a dominant role among the four components of the deep level set loss.  This superiority stems from its ability to capture regional statistical characteristics for segmentation, effectively addressing the discrete error points that may arise from the mse loss in the original diffusion probabilistic model, which primarily focuses on pixel--wise error minimization.  The geodesic distance penalty term emerges as the second most influential component, demonstrating significant efficacy in suppressing background interference and optimizing spatial localization.  The area and length terms, while comparatively less impactful, contribute to error suppression by constraining the emergence of spurious regions through geometric regularization.  These terms inherently align with the level set formulation, proving particularly powerful for tasks requiring precise regional delineation.

As shown in Fig.\ref{fig:xiaorong}, we demonstrate the advanced performance of the proposed model by comparing the performance capabilities of different components on the same instance. For example, without the region or distance terms, the model identifies spurious small regions $Case\_129\_49$. Without the length and area terms, the segmentation edges lack smoothness $Case\_129\_49$, $Case\_{330}\_{42}$, and the results fail to cover the target accurately $Case\_{008}\_{40}$, $Case\_{330}\_{42}$. Furthermore, a comparison between the final column (the proposed) and the sixth column (w/o PAR) reveals that the pixel--adaptive refinement module effectively corrects the final level set evolution results. The inclusion of the PAR module yields substantial performance improvements, primarily attributed to its capacity to enhance edge refinement by mitigating boundary roughness. This mechanism operates through neighborhood--aware feature modulation, effectively preserving local consistency while maintaining global structural integrity. The synergistic integration of all components establishes a robust framework that balances regional statistical modeling with geometric constraints, advancing the state-of-the-art in complex segmentation scenarios.

\section{Conclusion}

In this paper, we introduced TA-LSDiff, a novel framework that addresses the persistent challenges of automatic pancreas segmentation, namely its low contrast, small size, and complex topological variations. By fusing the semantic power of diffusion probabilistic models with the geometric interpretability of variational level set methods, our approach provides a robust and accurate segmentation solution.
The core of our contribution is a principled, topology--aware segmentation process. This process is motivated by our key theoretical insight: the equivalence between the Chan-Vese model's $L_2$ gradient flow and the boundary's topological derivative. Based on this, we designed an energy--guided diffusion framework where a four--term level set energy functional guides the learning process by being injected directly into the reverse diffusion step to steer the segmentation towards geometrically and topologically sound results. Furthermore, we introduced a pixel--adaptive refinement (PAR) module to enhance local consistency and produce smooth, coherent boundaries by leveraging neighborhood affinity.
Extensive experiments on four challenging public datasets (AbdomenCT-1K, NIH, MSD, and WORD) validate the superiority of TA-LSDiff. Our method achieves state-of-the-art performance, yielding Dice scores of 91.05±4.55\%, 87.93±3.27\%, 86.37±6.46\%, and 85.36±7.63\%, respectively, outperforming existing methods in both accuracy and stability. Ablation studies further confirmed the efficacy of our design, demonstrating that each component of the level set energy contributes to the final performance, with the region--fitting term playing a dominant role. The inclusion of the PAR module was also shown to be critical for refining boundary details. Importantly, TA-LSDiff achieves significant accuracy gains, and future research will focus on optimizing its computational efficiency through architectural improvements while maintaining this state-of-the-art performance.

\appendices

\section{Derivation of the Topological Derivative for Energy Functionals}
\label{app:td_derivations}
This appendix provides a detailed derivation for the topological derivative of the energy functionals used in our work. We first detail the proof for the classic Chan-Vese (CV) model and then present a proof for the more general two--phase regional energy functional, which our method employs.

\subsection{Boundary consistency of CV descent and topological drive.}
\label{app:cv_proof}

\textit{Proposition 1:} For the Chan-Vese data term, the $L_2$ gradient flow evolution is proportional to its boundary topological drive.
\[
\mathcal{T}_{\mathrm{CV}}(x) = -(f(x_0)-c_1)^2 + (f(x_0)-c_2)^2.
\]

\begin{proof}
The topological derivative  is defined as the first--order variation of the energy functional with respect to the nucleation of an infinitesimal hole \(B_{\rho,x}\) at a point \(x_0\). For the CV energy, this corresponds to moving the ball \(B_{\rho,x}\) from the inner region \(\Omega_1\) to the outer region \(\Omega_2\). This perturbation affects the regional means \(c_1\) and \(c_2\). To the first order, the perturbed means \(c_1^\rho\) and \(c_2^\rho\) are given by:
\begin{align*}
c^{\rho}_{1}
&=\frac{\int_{\Omega_{1}\setminus B_{\rho,x}} f\,\mathrm{d}x}{|\Omega_{1}|-\pi\rho^2}
= c_{1}-\frac{\int_{B_{\rho,x}}(f-c_{1})\,\mathrm{d}x}{|\Omega_{1}|-\pi\rho^2},\\
c^{\rho}_{2}
&=\frac{\int_{\Omega_{2}\cup B_{\rho,x}} f\,\mathrm{d}x}{|\Omega_{2}|+\pi\rho^2}
= c_{2}+\frac{\int_{B_{\rho,x}}(f-c_{2})\,\mathrm{d}x}{|\Omega_{2}|+\pi\rho^2}.
\end{align*}
The topological derivative is then calculated by substituting these perturbed means into the definition of the CV energy functional and taking the limit as \(\rho \to 0\):
$$
\begin{aligned}
&\mathcal{T}_{\mathrm{CV}}(x) = \lim_{\rho\to 0}\frac{1}{\pi\rho^{2}}\Bigg\{
\int_{\Omega_{1}\setminus B_{\rho,x}}(f-c_{1}^{\rho})^{2}\,\mathrm{d}x-\int_{\Omega_{1}}(f-c_{1})^{2}\,\mathrm{d}x \\
&\qquad\qquad\quad  + \int_{\Omega_{2}\cup B_{\rho,x}}(f-c_{2}^{\rho})^{2}\,\mathrm{d}x-\int_{\Omega_{2}}(f-c_{2})^{2}\,\mathrm{d}x \Bigg\} \\[4pt]
&= \lim_{\rho\to 0}\frac{1}{\pi\rho^{2}}\Bigg\{
\int_{\Omega_{1}\setminus B_{\rho,x}}(f-c_{1})^{2}\,\mathrm{d}x
-\frac{\Big(\int_{B_{\rho,x}}(f-c_{1})\,\mathrm{d}x\Big)^{2}}{|\Omega_{11}|-\pi\rho^{2}}\\
&\qquad\qquad\quad -\int_{\Omega_{1}}(f-c_{1})^{2}\,\mathrm{d}x + \int_{\Omega_{2}\cup B_{\rho,x} }(f-c_{2})^{2}\,\mathrm{d}x \\
&\qquad\qquad\quad -\frac{\Big(\int_{B_{\rho,x}}(f-c_{2})\,\mathrm{d}x\Big)^{2}}{|\Omega_{2}|+\pi\rho^{2}}
-\int_{\Omega_{2}}(f-c_{2})^{2}\,\mathrm{d}x \Bigg\} \\[4pt]
&= \lim_{\rho\to 0}\frac{1}{\pi\rho^{2}}\Bigg\{
-\int_{B_{\rho,x}}(f-c_{11})^{2}\,\mathrm{d}x
+ \int_{B_{\rho,x}}(f-c_{12})^{2}\,\mathrm{d}x \\
&\qquad-\frac{\Big(\int_{B_{\rho,x}}(f-c_{11})\,\mathrm{d}x\Big)^{2}}{|\Omega_{11}|-\pi\rho^{2}}
-\frac{\Big(\int_{B_{\rho,x}}(f-c_{12})\,\mathrm{d}x\Big)^{2}}{|\Omega_{12}|+\pi\rho^{2}}
\Bigg\} \\[4pt]
&= -\,(f(x_0)-c_{11})^{2} + (f(x_0)-c_{12})^{2}.
\end{aligned}
$$
The result is precisely the topological drive \(\mathcal{T}_{\mathrm{CV}}(x)\). This proves that the direction of steepest energy descent under a topological perturbation at the boundary is identical to the direction of the CV model's $L_2$ gradient flow.

\end{proof}

\subsection{Boundary consistency of two--phase regional energy descent and topological drive.}
\label{app:region_proof}
We now generalize the concept to the two--phase regional energy functional used in our method, which is based on a Gaussian statistical model for each region.

\textit{Proposition 2:}  Let the total energy functional be the sum of energies from two disjoint domains:
$$
F(\Omega_1, \Omega_2) = F_1(\Omega_1) + F_2(\Omega_2) = \int_{\Omega_1} e_1(x)\,\mathrm{d}x + \int_{\Omega_2} e_2(x)\,\mathrm{d}x.
$$
Here, $e_i(x) = \log(\sigma_i^2) + \frac{(f(x)-c_i)^2}{\sigma_i^2}$ is the negative log--likelihood, with $(c_i, \sigma_i^2)$ being the mean and variance of an intensity function $f(x)$ over the domain $\Omega_i$. The topological derivative, corresponding to moving an infinitesimal region $B_{\rho,x}$ from $\Omega_1$ to $\Omega_2$ at a point $x_0$, is given by:
\[
\mathcal{T}_{\mathrm{F_\Omega}}(x) = - e_1(x_0) + e_2(x_0) .
\]

\begin{proof}
The proof proceeds by analyzing the energy change resulting from the perturbation. The new domains are $\Omega'_1 = \Omega_1 \setminus B_{\rho,x}$ and $\Omega'_2 = \Omega_2 \cup B_{\rho,x}$. The total change in energy is the sum of the changes in each domain:
$$
\Delta F = \Delta F_1 + \Delta F_2 = [F_1(\Omega'_1) - F_1(\Omega_1)] + [F_2(\Omega'_2) - F_2(\Omega_2)].
$$
We analyze each component separately. 

\textbf{Step 1: Simplification of the Energy Functional}

We first simplify the integral form of the functional $F(\Omega_i)$. By distributing the integral and noting that $c_i$ and $\sigma_i^2$ are constant with respect to the integration variable $x$, we have:
\begin{align*}
F(\Omega_i) &= \int_{\Omega_i} \left( \log(\sigma_i^2) + \frac{(f(x)-c_i)^2}{\sigma_i^2} \right) \,\mathrm{d}x \\
&= \log(\sigma_i^2) \int_{\Omega_i} \mathrm{d}x + \frac{1}{\sigma_i^2} \int_{\Omega_i} (f(x)-c_i)^2 \,\mathrm{d}x \\
&= |\Omega_i|\log(\sigma_i^2) + \frac{1}{\sigma_i^2} \int_{\Omega_i} (f(x)-c_i)^2 \,\mathrm{d}x.
\end{align*}
By the definition of variance, $\sigma_i^2 = \frac{1}{|\Omega_i|} \int_{\Omega_i} (f(x)-c_i)^2 \,\mathrm{d}x$. Substituting this into the second term yields:
$$
F(\Omega_i) = |\Omega_i|\log(\sigma_i^2) + \frac{1}{\sigma_i^2} (|\Omega_i|\sigma_i^2) = |\Omega_i|(\log(\sigma_i^2) + 1).
$$
This simplified form is used for the subsequent analysis.

\vspace{1em}

\textbf{Step 2: Energy Change from Region Removal in $\Omega_1$}

When a small region $B_{\rho,x}$ (centered at $x_0$ with area $|B_{\rho,x}| \to 0$) is removed, the domain becomes $\Omega_1' = \Omega_1 \setminus B_{\rho,x}$. The change in the energy functional, $\Delta F_1$, is:
\begin{align*}
\Delta F_1 &= F_1(\Omega_1') - F_1(\Omega_1) \\
&= (|\Omega_1| - |B_{\rho,x}|)(\log((\sigma_1')^2) + 1) - |\Omega_1|(\log(\sigma_1^2) + 1) \\
&= |\Omega_1|\log\left(\frac{(\sigma_1')^2}{\sigma_1^2}\right) - |B_{\rho,x}|(\log((\sigma_1')^2) + 1).
\end{align*}

\vspace{1em}
\textbf{Step 3: First--Order Approximation of the New Variance $(\sigma_1')^2$}

To evaluate $\Delta F_1$, we need the first--order relationship between the new variance $(\sigma_1')^2$ and the original variance $\sigma_1^2$. This requires a Taylor expansion for the new mean $c_1'$ and $(\sigma_1')^2$ in terms of the small perturbation $|B_{\rho,x}|$. A detailed algebraic manipulation is required, ignoring terms of $O(|B_{\rho,x}|^2)$ and higher.
Simplified variance $\sigma^2$:
\[
\sigma_1^{2} \;=\; \frac{1}{\Omega_1}\!\int_{\Omega_1} (f(x)-c_1)^{2}\,\mathrm{d}x 
\;=\; \frac{1}{\Omega_1}\!\int_{\Omega_1} f(x)^{2}\,\mathrm{d}x - c_1^{2}.
\]
Updated second moment and variance $\sigma_1'^2$. Since $\int_{\Omega_1} f(x)^{2} = |\Omega_1|(\sigma_1^{2}+c_1^{2})$,
\[
\begin{aligned}
\int_{\Omega_1'} f(x)^{2} 
& \;=\; \int_{\Omega_1} f(x)^{2} - \int_{B_{\rho,x}} f(x)^{2}\\
&\;\approx\; |\Omega_1|(\sigma_1^{2}+c_1^{2}) - f(x_{0})^{2}\,|B_{\rho,x}| .
\end{aligned}
\]

\[
\begin{aligned}
{\sigma_1^{'}}^{2} & \;=\; \frac{1}{\Omega_1^{'}}\!\int_{\Omega_1^{'}} f(x)^{2}\,\mathrm{d}x - {c_1^{'}}^{2}\\
&\;\approx\; \frac{1}{\Omega_1^{'}}\Big(|\Omega|\Big(\sigma^{2}+c^{2}\Big) - |B_{\rho,x}|f(x_{0})\Big) \\
& \quad- \Big(\frac{|\Omega_1|c_1-|B_{\rho,x}|f(x_{0})}{|\Omega_1|-|B_{\rho,x}|}\Big)^{2}.
\end{aligned}
\]
Expanding and neglecting $O((b/m)^{2})$ terms gives

$$\sigma_1'^2 \;\approx\; \sigma_1^{2} \;-\; \frac{|B_{\rho,x}|}{|\Omega_1|}\,
\Big( (f(x_{0})-c_1)^{2} - \sigma_1^{2} \Big)
\;. \;$$

\vspace{1em}
\textbf{Step 4: Calculation of the Topological Derivative}

The topological derivative is defined as the limit of the rate of change:
$$
\begin{aligned}
\mathcal{T}_{\mathrm{F_{\Omega_1}}}(x) &= \lim_{|B_{\rho,x}| \to 0} \frac{\Delta F_1}{|B_{\rho,x}|}\\
&= \lim_{|B_{\rho,x}| \to 0} \left[ \frac{|\Omega_1|}{|B_{\rho_1,x}|}\log\left(\frac{(\sigma_1')^2}{\sigma_1^2}\right) - (\log((\sigma_1')^2) + 1) \right].
\end{aligned}
$$
We handle the logarithm using the first--order Taylor expansion $\log(1+x) \approx x$:
$$
\log\left(\frac{(\sigma_1')^2}{\sigma_1^2}\right) = \log\left(1 + \frac{(\sigma_1')^2 - \sigma_1^2}{\sigma_1^2}\right) \approx \frac{(\sigma_1')^2 - \sigma_1^2}{\sigma_1^2}.
$$
Substituting the approximation for $(\sigma_1')^2$ from Step 3:
$$
\begin{aligned}
\log\left(\frac{(\sigma_1')^2}{\sigma_1^2}\right) &\approx \frac{1}{\sigma_1^2} \left[ -\frac{|B_{\rho,x}|}{|\Omega_1|}\left( (f(x_0)-c_1)^2 - \sigma_1^2 \right) \right] \\
&= -\frac{|B_{\rho,x}|}{|\Omega_1|\sigma_1^2}\left( (f(x_0)-c_1)^2 - \sigma_1^2 \right).
\end{aligned}
$$
Now, we substitute this result back into the expression for $\Delta F_1$:
\begin{align*}
\Delta F_1 &\approx |\Omega| \left[ -\frac{|B_{\rho_1,x}|}{|\Omega_1|\sigma_1^2}\left( (f(x_0)-c_1)^2 - \sigma_1^2 \right) \right]\\
& \quad - |B_{\rho_1,x}|(\log(\sigma_1^2) + 1) \\
&= -\frac{|B_{\rho,x}|}{\sigma_1^2}\left( (f(x_0)-c_1)^2 - \sigma_1^2 \right) - |B_{\rho,x}|(\log(\sigma_1^2) + 1) \\
&= -|B_{\rho,x}| \left[ \frac{(f(x_0)-c_1)^2}{\sigma_1^2} - 1 + \log(\sigma_1^2) + 1 \right] \\
&= -|B_{\rho,x}| \left[ \log(\sigma_1^2) + \frac{(f(x_0)-c_1)^2}{\sigma_1^2} \right].
\end{align*}
Recalling the definition $e_1(x_0) = \log(\sigma_1^2) + \frac{(f(x_0)-c_1)^2}{\sigma_1^2}$, we have:
$$
\Delta F_1 \approx -|B_{\rho,x}| \cdot e_1(x_0).
$$

\vspace{1em}
\textbf{Step 5: Energy Change from Region Addition to $\Omega_2$}

This part analyzes the effect of adding the region $B_{\rho,x}$ to $\Omega_2$. The change in energy $\Delta F_2$ is:
$$
\begin{aligned}
\Delta F_2 &= F_2(\Omega_2 \cup B_{\rho,x}) - F_2(\Omega_2) \\
&= (|\Omega_2| + |B_{\rho,x}|)(\log((\sigma'_2)^2) + 1) - |\Omega_2|(\log(\sigma_2^2) + 1).
\end{aligned}
$$
This simplifies to:
$$
\Delta F_2 = |\Omega_2|\log\left(\frac{(\sigma'_2)^2}{\sigma_2^2}\right) + |B_{\rho,x}|(\log((\sigma'_2)^2) + 1).
$$
The first--order approximation for the new variance $(\sigma'_2)^2$ when adding a region is given by:
$$
(\sigma'_2)^2 \approx \sigma_2^2 + \frac{|B_{\rho,x}|}{|\Omega_2|}\left( (f(x_0)-c_2)^2 - \sigma_2^2 \right).
$$
Using the Taylor expansion $\log(1+x) \approx x$, we find:
$$
\log\left(\frac{(\sigma'_2)^2}{\sigma_2^2}\right) \approx \frac{(\sigma'_2)^2 - \sigma_2^2}{\sigma_2^2} \approx \frac{|B_{\rho,x}|}{|\Omega_2|\sigma_2^2}\left( (f(x_0)-c_2)^2 - \sigma_2^2 \right).
$$
Substituting this back into the expression for $\Delta F_2$:
\begin{align*}
\Delta F_2 &\approx |\Omega_2| \left[ \frac{|B_{\rho,x}|}{|\Omega_2|\sigma_2^2}\left( (f(x_0)-c_2)^2 - \sigma_2^2 \right) \right] \\
& \quad + |B_{\rho,x}|(\log(\sigma_2^2) + 1) \\
&= \frac{|B_{\rho,x}|}{\sigma_2^2}\left( (f(x_0)-c_2)^2 - \sigma_2^2 \right) + |B_{\rho,x}|(\log(\sigma_2^2) + 1) \\
&= |B_{\rho,x}| \left[ \frac{(f(x_0)-c_2)^2}{\sigma_2^2} - 1 + \log(\sigma_2^2) + 1 \right] \\
&= |B_{\rho,x}| \left[ \log(\sigma_2^2) + \frac{(f(x_0)-c_2)^2}{\sigma_2^2} \right] \approx |B_{\rho,x}| \cdot e_2(x_0).
\end{align*}

\vspace{1em}
\textbf{Step 6: Total Change and Final Result}

Combining the results from Step 4 and Step 5, the total change in energy is:
$$
\Delta F = \Delta F_1 + \Delta F_2 \approx -|B_{\rho,x}| \cdot e_1(x_0) + |B_{\rho,x}| \cdot e_2(x_0).
$$
The topological derivative is the limit of the rate of this change as $|B_{\rho,x}| \to 0$:
$$
\begin{aligned}
\mathcal{T}_{\mathrm{F_\Omega}}(x) &= \lim_{|B_{\rho,x}| \to 0} \frac{\Delta F}{|B_{\rho,x}|} \\
&= \lim_{|B_{\rho,x}| \to 0} \frac{|B_{\rho,x}|(-e_1(x_0,) + e_2(x_0))}{|B_{\rho,x}|} \\
&= -e_1(x_0) + e_2(x_0).
\end{aligned}
$$
This completes the proof.
\end{proof}

\section*{References}
\bibliographystyle{IEEEtran}
\bibliography{references}  
\end{document}